\documentclass[letterpaper]{article} 
\usepackage{aaai25}  
\usepackage{times}  
\usepackage{helvet}  
\usepackage{courier}  
\usepackage[hyphens]{url}  
\usepackage{graphicx} 
\usepackage{natbib}  
\usepackage{caption} 
\usepackage{algorithm}
\usepackage{algorithmic}

\usepackage{newfloat}
\usepackage{listings}
\DeclareCaptionStyle{ruled}{labelfont=normalfont,labelsep=colon,strut=off} 
\floatstyle{ruled}
\newfloat{listing}{tb}{lst}{}
\floatname{listing}{Listing}
\usepackage{latexsym}
\usepackage{amssymb}
\usepackage{amsmath}
\usepackage{amsthm}
\usepackage{booktabs}
\usepackage{enumitem}
\usepackage{graphicx}
\usepackage{color}
\usepackage[table]{xcolor}
\definecolor{poliblue3}{RGB}{25,43,67}
\usepackage{mathabx} 

\usepackage{mathtools}

\usepackage{dsfont}
\usepackage{thmtools}
\usepackage{thm-restate}

\usepackage{etoolbox}
\usepackage{bm}

\newtheorem{theorem}{Theorem}[section]
\newtheorem{lemma}[theorem]{Lemma}

\newtheorem{proposition}[theorem]{Proposition}

\theoremstyle{remark}

\usepackage{tabularx}
\usepackage{multirow}
\usepackage{makecell}
\usepackage{arydshln} 

\usepackage{pifont}

\newcommand{\cmark}{\textcolor{green}{\ding{51}}} 
\newcommand{\xmark}{\textcolor{red}{\ding{55}}} 

\title{Efficient Learning of POMDPs with Known Observation Model in Average-Reward Setting}
\author {
    Alessio Russo\textsuperscript{\rm 1},
    Alberto Maria Metelli\textsuperscript{\rm 1},
    Marcello Restelli\textsuperscript{\rm 1}
}
\affiliations {
    \textsuperscript{\rm 1}Politecnico di Milano\\
    alessio.russo@polimi.it, albertomaria.metelli@polimi.it, marcello.restelli@polimi.it
}

\usepackage{bibentry}

\begin{document}

\maketitle

\begin{abstract}
Dealing with Partially Observable Markov Decision Processes is notably a challenging task. We face an average-reward infinite-horizon POMDP setting with an unknown transition model, where we assume the knowledge of the observation model. Under this assumption, we propose the Observation-Aware Spectral (OAS) estimation technique, which enables the POMDP parameters to be learned from samples collected using a belief-based policy. Then, we propose the OAS-UCRL algorithm that implicitly balances the exploration-exploitation trade-off following the \emph{optimism in the face of uncertainty} principle. The algorithm runs through episodes of increasing length. For each episode, the optimal belief-based policy of the estimated POMDP interacts with the environment and collects samples that will be used in the next episode by the OAS estimation procedure to compute a new estimate of the POMDP parameters. Given the estimated model, an optimization oracle computes the new optimal policy. 
We show the consistency of the OAS procedure, and we prove a regret guarantee of order \(\mathcal{O}(\sqrt{T \log(T)})\) for the proposed OAS-UCRL algorithm. We compare against the oracle playing the optimal stochastic belief-based policy and show the efficient scaling of our approach with respect to the dimensionality of the state, action, and observation space. We finally conduct numerical simulations to validate and compare the proposed technique with other baseline approaches.
\end{abstract}

%

\section{Introduction}
In Reinforcement Learning (RL)~\cite{sutton1998}, an agent interacts with an unknown or partially known environment to maximize the expected long-term reward. This class of approaches has been successfully used in a large variety of problems~\cite{mnih2016async, schulman2017PPO, casas2017deep}. Most existing results consider learning under \emph{fully observable} systems where the underlying state of the system is visible to the learner. However, many real-world problems present a \emph{partially observable} nature, further complicating the task since an agent needs to implement a memory mechanism and keep a belief over the system state. Some applications include autonomous driving~\cite{levinson2011towards}, medical diagnosis~\cite{hauskrecht2000PlanningTO}, resource allocation~\cite{gilbert2005resource}.\\
Dealing with Partially Observable MDPs (POMDP)~\cite{sondik1978optimal} is both statistically and computationally intractable in general~\cite{mossel2005learning, mundhenk2000complexity}.\\
Some recent works have focused on considering a class of POMDPs where sample-efficient algorithms can be devised. 
\citet{Azizzadenesheli2016Reinforcement} and~\citet{guo2016pac} extend to the decision-making problem the use of techniques based on \emph{spectral decomposition} (SD), which have been initially devised for estimation of latent variable models such as HMMs~\cite{Anandkumar2014Tensor}. \citet{jin2020sample} presents a sample-efficient algorithm for undercomplete POMDPs, where the number of observations is larger than the number of states. All these classes fall under the richer family of \emph{weakly-revealing} POMDPs, which have been lately defined in the work of~\citet{liu2022partially}.

This paper considers a POMDP setting where the \emph{observation model is known}, and the agent needs to learn the \emph{unknown transition model}. The assumption of having knowledge of some portions of the system has been previously used in problems with partial observability. Examples are the works on Latent Bandits~\cite{MaillardLatent2014, hongLatent2020} and Switching Latent Bandits~\cite{anonymous}, where the reward distributions of each bandit are known, but the identity of the active bandit (latent state) is hidden to the learner.\\ 
In many real-world applications, situations where the observation model is known emerge. This happens, for example, when the learning phase of the process is decoupled into an offline phase (used for learning the observation model) and an online phase where the knowledge of the dynamics of the underlying state is hidden. Examples of this type are applications that initially learn through simulators~\cite{thanan2021resource, alonso2017virtual}.
Furthermore, this setting is of practical interest in non-stationary environments where changes do not affect the entire environment but only the dynamics of the system. In these cases, previous knowledge of the observation model can be retained and employed for faster learning of the system dynamics. Analogously, this condition is applicable in the field of \emph{transfer learning} where the knowledge of some domain can be used to help learning on a similar domain.

\paragraph{Contributions.}The contributions of the paper are summarized as follows:
\begin{itemize}[noitemsep, leftmargin=*, topsep=0pt]
    \item We propose a consistent procedure that estimates the transition model of a POMDP in the average-reward infinite-horizon setting under the assumption of knowing the observation model. In particular, by assuming the ergodicity of the Markov chain induced by the belief-based policy used to collect samples, we prove that our OAS procedure provides an estimation rate of \(\mathcal{O}(1/\sqrt{T})\) while also showing a desirable dependency on the problem parameters;
    \item We plug this approach into a regret minimization algorithm and, by exploiting the obtained convergence result, we establish that our OAS-UCRL  algorithm achieves a regret bound of \(\widetilde{\mathcal{O}}(\sqrt{T})\)\footnote{The notation with \(\widetilde{\mathcal{O}}(\cdot)\) ignores logarithmic terms.} under the set of stochastic belief-based policies;
    \item We provide numerical simulations validating both our OAS estimation procedure and showing the effectiveness of the OAS-UCRL algorithm over state-of-the-art approaches.
\end{itemize}

\section{Related Works}\label{sec:relatedWorks}
Learning in the fully observable RL setting has been extensively studied in the last few years. Some works considered the episodic finite-horizon setting~\cite{azar2017MinimaxRB, jin2018qlearning, zanette2020LearningNO}, while a different line of works studied the non-episodic undiscounted setting~\cite{bartlett2009regal, jaksch2010near, ortner2012e}. Differently, the study of the partially observable RL setting has been relatively less explored.\\
Learning in the general POMDP setting is known to be statistically hard~\cite{krishnamurthy2016PACRL}, and the hard instances are generally those where observations do not contain useful information for identifying the system dynamics.

\paragraph{Efficient Learning in POMDPs.}
Some recent works have focused on specific POMDP classes where the observation model is full-rank or have considered its robust version called \(\alpha\)-weakly revealing where \(\alpha\) intuitively quantifies the amount of information that observations provide when inferring the latent states. These conditions rule out the pathological POMDP instances that can lead to hardness results.\\
A first line of work focuses on the \textbf{episodic} setting~\cite{guo2016pac, jin2020sample, liu2022partially, liu2022optimistic}. \citet{jin2020sample} propose an algorithm with optimal sample complexity \((1/\epsilon^2)\) for finding an \(\epsilon\)-optimal policy for the undercomplete case, where the number of latent states must not be greater than the number of observations (\(S\le O\)), while \citet{liu2022partially} present a new generic algorithm based on \emph{Maximum Likelihood Estimation} that reaches a \(\widetilde{\mathcal{O}}(\sqrt{K})\) regret for the undercomplete setting, with \(K\) being the number of episodes. They are also able to handle the more difficult overcomplete case (\(S > O\)), for which they show a regret of order \(\widetilde{\mathcal{O}}(K^{2/3})\).\\
A second line of works instead focuses on the \emph{non-episodic} average reward setting~\cite{ xiong2022sublinear, jahromi2022online, Azizzadenesheli2016Reinforcement}. Among them,~\citet{xiong2022sublinear} employ assumptions similar to ours and devise an algorithm alternating exploration and exploitation phases, they make use of spectral decomposition techniques to learn the POMDP parameters from samples collected during the exploration phase. Their algorithm suffers a regret of \(\mathcal{O}(T^{2/3})\) with respect to the optimal POMDP policy.\\
~\citet{jahromi2022online} propose a posterior sampling-based algorithm for which they prove a Bayesian regret of order \(\mathcal{O}(T^{2/3})\) under some technical assumptions on the belief state approximation and the transition parameters estimation. Differently,~\citet{Azizzadenesheli2016Reinforcement} adapt spectral decomposition techniques for learning the POMDP parameters under stochastic memoryless policies and present an algorithm suffering a \(\widetilde{\mathcal{O}}(\sqrt{T})\) regret bound. However, they compare against an oracle using the best stochastic memoryless policy, which is weaker than the optimal POMDP oracle used in the previously mentioned works.\\
For a thorough comparison with the most relevant works, we refer the reader to Table~\ref{tab:comparisonRL} in Appendix~\ref{appendix:comparison}.

\paragraph{Settings with Known Observation Model.} 
Having partial knowledge of the model is a relatively common assumption in RL:~\cite{azar2017MinimaxRB} for example, assume to know the reward model characterizing the environment. Concerning instead the partial observable setting, this assumption has been largely used in MAB problems, such as Latent Bandits~\cite{MaillardLatent2014, hongLatent2020} where the learning agent knows reward distributions of arms conditioned on an unknown discrete latent, or its non-stationary variants~\cite{hongNonStationary2020, anonymous} where the reward distributions are still known but the latent state keeps changing according to an unknown Markov chain. Analogously to our case, for the POMDP setting,~\citet{jahromi2022online} develop a posterior sampling-based algorithm under the assumption of knowing the observation model.

\section{Preliminaries}\label{sec:preliminaries}
In this section, we provide the necessary backgrounds and
notations that will be used throughout the rest of the article.

\paragraph{Controlled Markov Process.} A Controlled Markov Process (CMP) is a tuple \((\mathcal{S}, \mathcal{A}, \mathbb{T}, \mu)\), where \(\mathcal{S}\) is a finite state space (\(|\mathcal{S}| \eqqcolon S\)), \(\mathcal{A}\) is a finite action space (\(|\mathcal{A}|\eqqcolon A\)), \(\mathbb{T}=\{\mathbb{T}_a\}_{a \in \mathcal{A}}\) is a matrix of dimension \(SA \times S\) and represents a collection of transition matrices \(\mathbb{T}_a\) with size \(S\times S\) such that \(\mathbb{T}_a(\cdot|s) \in \Delta(\mathcal{S})\) denotes the distribution of the next state when the agent takes action \(a \in \mathcal{A}\) in state \(s \in \mathcal{S}\), and \(\bm{\mu} \in \Delta(\mathcal{S})\) is the distribution of the initial state.\footnote{We denote with \(\Delta(\mathcal{X})\) the simplex over a finite set \(\mathcal{X}\).}

\paragraph{Policies and State Distributions in CMPs.} A policy \(\pi\) defines the behavior of an agent interacting in an environment. It is characterized by a sequence of decision rules \(\pi:= (\pi_t)_{t=0}^\infty\). Each of them maps a history \(h \in \mathcal{H}_t\) into distributions over actions \(\pi_t: \mathcal{H}_t \rightarrow \Delta(\mathcal{A})\) such that \(\pi_t(a|h)\) defines the conditional probability of taking action \(a \in \mathcal{A}\) having observed history \(h \in \mathcal{H}_t\). We denote as \(\mathcal{H}\) the space of histories of arbitrary length. When interacting with a CMP the history can be defined as \(h:=(s_j,a_j)_{j=0}^t \in \mathcal{H}_t\).\\
We denote with \(\Pi\) the set of all the policies. A policy \(\pi \in \Pi\) interacting with a CMP induces a state visitation distribution \(d_t^\pi(s) \coloneqq P(S_t=s|\pi)\)~\cite{puterman2014discrete}, also defined as \(d_t^\pi(s)=\sum_{s' \in \mathcal{S}} \sum_{a' \in \mathcal{A}} d_{t-1}^\pi(a',s')\mathbb{T}_{a'}(s|s')\), 
where \(\bm{d}_t^\pi \in \Delta(\mathcal{A}\times \mathcal{S})\) is the \(t\)-step action-state visitation distribution and is defined as \(d_t^\pi(a,s) \coloneqq P(A_t=a, S_t=s|\pi)\). The fixed point of the defined temporal relation (when it exists) is \(d_\infty^\pi(s) \coloneqq \lim_{t \to \infty} d_t^\pi(s)\) and is called \emph{stationary state distribution}. If the limiting distribution \(d_\infty^\pi(s)\) exists and \(d_\infty^\pi(s) > 0\) for all \(s \in \mathcal{S}\), then this distribution is unique. Analogous considerations hold for the {stationary action-state distribution} \(\bm{d}_\infty^\pi \in \Delta(\mathcal{A} \times \mathcal{S})\).\\
From now on, we will use vector notation \(\bm{d}^{\pi}_S \in \Delta(\mathcal{S})\) for the stationary state distribution and \(\bm{d}^{\pi}_{AS} \in \Delta(\mathcal{A} \times \mathcal{S})\) for the stationary action-state distribution.\footnote{The subscript represents the support size of the distribution.}

\paragraph{Partially Observable MDP.} A Partially Observable Markov Decision Process (POMDP)~\cite{astrom1965optimal} is described by a tuple \(\mathcal{Q}\coloneqq(\mathcal{S}, \mathcal{A}, \mathcal{O}, \mathbb{T}, \mathbb{O}, \mu, r)\) where \(\mathcal{S}\), \(\mathcal{A}\), \(\mathbb{T}\), \(\mu\) are denoted as in the CMP. \(\mathcal{O}\) denotes the finite space of observation (\(|\mathcal{O}|\eqqcolon O\)); \(\mathbb{O}=\{\mathbb{O}_a\}_{a \in \mathcal{A}}\) denotes the set of emission matrices of size \(O \times S\) such that \(\mathbb{O}_a(\cdot|s) \in \Delta(\mathcal{O})\) gives the distribution over observations when the agent takes action \(a\in \mathcal{A}\) conditioned on the hidden state \(s \in\mathcal{S}\); \(r: \mathcal{O} \rightarrow [0,1]\) is the known reward function that deterministically maps each observation to a finite reward such that \(r(o)\) is the received reward when the agent observes \(o \in \mathcal{O}\).
In a POMDP, states are hidden, and the agent can only see the observations and its actions. The interaction proceeds as follows. At each step \(t\), the agent takes an action \(a_t \in \mathcal{A}\) and receives an observation \(o_t \in \mathcal{O}\) conditioned on the hidden state \(s_t \in \mathcal{S}\) according to the law \(\mathbb{O}_{a_t}(\cdot|s_t)\). Finally, the action played makes the POMDP transition into a new hidden state \(s_{t+1}\) according to the distribution \(\mathbb{T}_{a_t}(\cdot|s_t)\).\\

\paragraph{From POMDP to Belief MDP.} By knowing the transition and emission matrices \(\mathbb{T}\) and \(\mathbb{O}\) respectively, and from the observed history at time \(t-1\), \(h_{t-1}:= (a_j, o_j)_{j=0}^{t-1}\), it is possible to build a belief vector \(b_t \in \mathcal{B}\) with \(\mathcal{B}:=\Delta(\mathcal{S})\) defined as a simplex over the state space. At time \(t\), we define \(b_t(s):=P(S_t=s|h_{t-1})\). From the agent's point of view, a POMDP can be seen as a belief MDP~\cite{krish2016partially}. The update rule of the belief \(b_{t+1}\) is determined by Bayes's theorem as follows:
\begin{align}\label{eq:beliefUpdate}
	b_{t+1}(s) = \frac{\sum_{s'}b_t(s') \mathbb{O}_{a_t}(O_t=o_t|S_t=s') \mathbb{T}_{a_t}(s|s')}{\sum_{s''} \mathbb{O}_{a_t}(O_t=o_t|S_t=s'') b_t(s'')}.
\end{align}
The average reward of the infinite-horizon belief MDP given an initial belief \(b \in \mathcal{B}\) is defined as: \(\rho_b^\pi:=\lim \sup_{T \to \infty} (1/T)\mathbb{E}[\sum_{t=1}^{T}r(o_t)|b_1=b]\). If the underlying MDP is weakly communicating,~\citet{bertsekas1995dynamic} showed that \(\rho^*:=\sup_\pi \rho_b^\pi\) is independent of the initial belief \(b\) and the following Bellman optimality equation can be defined:
\begin{align}\label{eq:Bellman}
	\rho^* + v(b) = \underset{a \in \mathcal{A}}{\max}\; \left[g(b,a) + \int_{\mathcal{B}}P(\,db'|b, a) v(b') \right],
\end{align}
where \(g(b,a)\) represents the expected instantaneous reward obtained when taking action \(a\in\mathcal{A}\) having belief \(b \in \mathcal{B}\) such that \(g(b,a) = \sum_{s \in \mathcal{S}} \sum_{o \in \mathcal{O}} b(s) \mathbb{O}_a(o|s)r(o)\). Instead, \(v: \mathcal{B} \to \mathbb{R}\) is the bias function that commonly appears in the average reward MDP setting~\cite{Mahadevan1996average}. 

\section{Problem Formulation}\label{sec:problemSetting}
We consider a non-episodic POMDP setting, as the one described in Section~\ref{sec:preliminaries}. Specifically, we focus on \emph{undercomplete} POMDPs~\cite{jin2020sample}, where the number of states is less than or equal to the number of observations \(S \le O\). Similarly to~\citet{jahromi2022online}, we assume knowledge of the emission matrices \(\mathbb{O}=\{\mathbb{O}_a\}_{a \in \mathcal{A}}\), while we learn the transition model \(\mathbb{T}=\{\mathbb{T}_a\}_{a \in \mathcal{A}}\).\\
We focus on stochastic belief-based policies which map the space \(\mathcal{B}\) of belief over the states to a distribution over actions, such that \(\pi: \mathcal{B} \to \Delta(\mathcal{A})\). We denote by \(\mathcal{P}\) the set of all stochastic belief-based policies having probability at least \(\iota\) (with \(\iota > 0\)) of exploring all actions:
\begin{align}\label{def:policySet}
	\mathcal{P} \coloneqq \{\pi: \underset{b \in \mathcal{B}}{\min}\: \underset{a \in \mathcal{A}}{\min}\: \pi(a|b) \ge \iota\}.
\end{align}
In the following, we denote the assumptions that we enforce for our setting.
\begin{restatable}[\textbf{Minimum Value Transition Matrices}]{assumption}{assMinElem}\label{ass:minElem}
	The smallest value appearing in the transition matrices is given by \(\epsilon := \underset{s,s'\in \mathcal{S}, a \in \mathcal{A}}{\min}\mathbb{T}_a(s'|s) > 0\).
\end{restatable}
This assumption is necessary in our setting for multiple purposes. From a technical point of view, it ensures geometric ergodicity~\cite{krish2016partially} under each policy \(\pi\), thus allowing it to reach its stationary distribution exponentially fast. Furthermore,~\citet{xiong2022sublinear} showed that under this assumption the diameter \(D\) of the belief MDP can be bounded and this allows Equation~\eqref{eq:Bellman} to be well defined.
Most importantly, this assumption is fundamental in the theoretical analysis to bound the error in the belief computed using the estimated transition matrices, using the result appearing in Proposition~\ref{prop:beliefErrorPOMDP} (see the appendix). Despite seeming a strong assumption, it is satisfied in relevant POMDP applications involving information gathering as detailed in~\cite{guo2016pac}. Furthermore, this assumption has been commonly used in works facing settings with partial observability ~\cite{zhou2021regime, anonymous, jiang2023online, xiong2022sublinear}.

\begin{restatable}[\textbf{\(\bm{\alpha}\)-weakly Revealing Condition}]{assumption}{assWeaklyRev}\label{ass:weaklyRev}
	There exists \(\alpha>0\) such that \(\underset{a \in \mathcal{A}}{\min} \; \sigma_S(\mathbb{O}_a) \ge \alpha\).
\end{restatable}
Here, we denote with \(\sigma_S(\mathbb{O}_a)\) the \(S\)-th singular value of matrix \(\mathbb{O}_a\).
This second assumption is related to the identifiability of the POMDP parameters and has been largely used in its \emph{full-rank} version in works using spectral decomposition techniques~\cite{Azizzadenesheli2016Reinforcement, zhou2021regime, hsu2012spectral}. This condition allows excluding hard POMDP instances and has been employed to define the tractable subclass of \emph{weakly revealing} POMDPs~\cite{jin2020sample, liu2022partially, liu2022optimistic} (see Section~\ref{sec:relatedWorks}). Moreover, it directly implies that \(S \le O\), which is a common scenario in many real-world applications. Just to name a few, in spoken dialogue systems, the number of observations is much larger than the state (meaning) of the conversation~\cite{png2012building}; or medical applications where the state (physical condition) of a patient generates a large number of different observations~\cite{hauskrecht2000PlanningTO}.

\begin{restatable}[\textbf{Policy Set}]{assumption}{policySet}\label{ass:policySet}
	The policy \(\pi\) interacting with the POMDP belongs to \(\mathcal{P}\).
\end{restatable}
This assumption guarantees that each action \(a \in \mathcal{A}\) is constantly chosen and is commonly used in spectral-based approaches. Indeed, we will see soon that our estimation procedure resembles SD methods.

~\citet{Azizzadenesheli2016Reinforcement} use a similar assumption when defining their policy class. In particular, they focus on a less powerful class that denotes the set of all stochastic memoryless policies having non-zero probabilities for all actions.\footnote{When using a memoryless policy, the choice of each action is conditioned only on the last received observation rather than the whole history.}

\paragraph{Learning Objective.} 
Having defined all the relevant elements for our setting, we are ready to reformulate the Bellman equation in~\eqref{eq:Bellman} as follows:
\begin{equation}\label{eq:ourBellmanEq}\resizebox{.9\linewidth}{!}{$\displaystyle
	\rho^* + v(b) = \underset{\pi \in \mathcal{P}}{\max} \underset{a \sim \pi(\cdot|b)}{\mathbb{E}}\left[g(b,a) + \int_{\mathcal{B}}P(\,db'|b, a) v(b') \right],$}
\end{equation}
where the maximization is over the policy class \(\mathcal{P}\). Under Assumption~\ref{ass:minElem}, this equation is always verified (see Proposition~\ref{prop:uniformBoundBias} in the appendix). 
Determining the optimal policy for the POMDP model is generally computationally intractable~\cite{madani1999computability}. In this work, we do not focus on solving this planning problem for a known model. Instead, we assume access to an optimization oracle capable of solving Equation~\eqref{eq:ourBellmanEq} and providing the optimal average reward \(\rho^*\) and the optimal stationary policy \(\pi\) for a given POMDP instance \(\mathcal{Q}\).
Our learning objective is to minimize the total regret after \(T\) periods. It is defined as: 
\begin{align}\label{def:regret}
	\mathcal{R}_T := T\rho^* - \sum_{t=1}^{T} r_t^\pi(o_t),
\end{align}
with \(r_t^\pi(o_t)\) being the reward obtained from the observation received while following policy \(\pi \in \mathcal{P}\).

\section{OAS Estimation Procedure}\label{sec:estimationProcedure}
The core idea behind the presented approach is to exploit the known relation between the observations and the underlying latent state. In the following, we will use the term ergodic to denote a Markov chain induced by a policy \(\pi \in \mathcal{P}\) having stationary distribution \(\bm{d}_S^{\pi} \in \Delta(\mathcal{S})\) such that the chain is irreducible and aperiodic and \(d_S^{\pi}(s) > 0\) for all states \(s \in \mathcal{S}\).

Before proceeding, we provide the following result, whose proof is reported in Appendix~\ref{appendix:propUniqueStat}:
\begin{restatable}[]{proposition}{uniqueStationaryPOMDP}\label{proposition:uniqueStationaryPOMDP}
	Let us assume that a policy \(\pi \in \mathcal{P}\) induces an ergodic Markov chain with stationary distribution \(\bm{d}^{\pi}_S \in \Delta(\mathcal{S})\) when interacting with a POMDP instance $\mathcal{Q}$. Then, there exists a unique stationary distribution \(\bm{d}_{A^2S^2}^{\pi}\in \Delta(\mathcal{A}^2 \times \mathcal{S}^2)\) over consecutive action-state pairs.\footnote{Formally, \(d_{A^2S^2}^{\pi}(a,a',s, s'):= \lim_{t \to \infty}d_t^\pi(a,a',s,s')\) with \(d_t^\pi(a,a',s,s'):= P(A_t=a, A_{t+1}=a', S_t=s, S_{t+1}=s'|\pi)\).}
\end{restatable}
Clearly, because of the partially observable nature of the considered setting, we are not able to directly estimate the induced distribution \(\bm{d}_{A^2S^2}\). However, we show how they are related to the received observations and how to use this information for estimation. For this purpose, we first need to define an analogous distribution \(d_t^\pi(a,o):= P(A_t=a, O_t=o|\pi)\) defined over action-observation pairs and its doubled version \(d_t^\pi(a,a',o,o'):= P(A_t=a, A_{t+1}=a', O_t=o, O_{t+1}=o'|\pi)\) considering consecutive action-observation pairs. We denote as well their limit versions as \(\bm{d}_{AO}^\pi \in \Delta(\mathcal{A} \times \mathcal{O})\) and \(\bm{d}_{A^2O^2}^\pi \in \Delta(\mathcal{A}^2 \times \mathcal{O}^2)\) respectively.\\
The relations holding between these quantities can be characterized as follows:
\begin{align*}
	d_{AO}^{\pi}(a,o) = \sum_{s \in \mathcal{S}} \mathbb{O}_a(o|s)d_{AS}^{\pi}(a,s) \quad \forall a \in \mathcal{A}, o \in \mathcal{O}.
\end{align*}
Similarly, \(\forall a,a' \in \mathcal{A}\) and \(\forall o,o' \in \mathcal{O}\), we have:
\begin{align}\label{eq:relationOS}
	& d_{A^2O^2}^{\pi}(a,a',o,o') = \notag \\ 
 & \qquad \sum_{s,s' \in \mathcal{S}}\mathbb{O}_a(o|s) \mathbb{O}_{a'}(o'|s')d_{A^2S^2}^{\pi}(a,a',s,s').
\end{align}
These equations link the probability of action-state pairs with that of action-observation pairs and turn out to be relevant for the estimation of the distribution \(\bm{d}_{A^2S^2}^{\pi} \in \Delta(\mathcal{A}^2 \times \mathcal{S}^2)\), as will be shown in the following subsection.

For the remainder of this section, we show how the estimates of the transition model \(\widehat{\mathbb{T}}_a\) can be computed using the estimated \(\widehat{\bm{d}}_{A^2S^2}^{\pi}\). We first need to marginalize over the second component such that:
\begin{align}\label{eq:fromAASStoASS}
	\widehat{d}_{AS^2}^{\pi}(a,s,s') = \sum_{a' \in \mathcal{A}} \widehat{d}_{A^2S^2}^{\pi}(a,a',s,s'),
\end{align}
where \(\widehat{\bm{d}}_{AS^2}^{\pi}\) is an estimate of \(\bm{d}_{AS^2}^{\pi} \in \Delta(\mathcal{A} \times \mathcal{S}^2)\)
representing the stationary distribution over the tuple \((a,s,s')\). Since it is immediate to see that \(\mathbb{T}_a(s'|s) \:\propto \: d_{AS^2}^{\pi}(a,s,s')\), we use this relation to obtain \(\widehat{\mathbb{T}}_a\) from \(\widehat{\bm{d}}_{AS^2}^{\pi}\).
Indeed, the final step results in normalizing the obtained quantities such that each row of the estimated transition matrix \(\widehat{\mathbb{T}}_a\) sums to one. For any \((a,s,s') \in \mathcal{A} \times \mathcal{S}^2\), we get:
\begin{align}\label{eq:fromASStoT}
	\widehat{\mathbb{T}}_a(s'|s) = \frac{\widehat{d}_{AS^2}^{\pi}(a,s,s')}{\sum_{s'' \in \mathcal{S}}\widehat{d}_{AS^2}^{\pi}(a,s,s'')}.
\end{align}

\subsection{Proposed Approach and Theoretical Guarantees}\label{subsec:estimationAlgorithm}
Let us now introduce the quantities that will be useful for estimating the stationary distribution of consecutive action-state pairs \(\bm{d}_{A^2S^2}^{\pi}\) induced by a policy \(\pi \in \mathcal{P}\). First of all, for each pair of actions \(a,a' \in \mathcal{A}\), we define the element:
\begin{align}\label{eq:fromSingleObsToDoubleObs}
	\mathbb{O}_{a,a'} := \mathbb{O}_{a} \otimes \mathbb{O}_{a'},
\end{align}
where symbol \(\otimes\) denotes the Kronecker product~\cite{loan2000kron} between the two matrices. The obtained matrix \(\mathbb{O}_{a,a'}\) has dimensions \(O^2 \times S^2\). We encode each pair of ordered observations and ordered states into variables \(i \in [O^2]\) and \(j \in [S^2]\), respectively. We have that for any variable \(i\) mapping a pair \((o,o')\) and any variable \(j\) mapping a pair \((s, s')\), it holds that: \(\mathbb{O}_{a,a'}(i, j)=\mathbb{O}_{a}(o|s)\mathbb{O}_{a'}(o'|s')\).\\ 
From the properties of the Kronecker product, we have that \(\forall a, a' \in \mathcal{A}\):
\begin{align*}
	\sigma_{\min}(\mathbb{O}_{a, a'}) = \sigma_{S^2}(\mathbb{O}_{a, a'}) = \sigma_S(\mathbb{O}_{a})\sigma_S(\mathbb{O}_{a'}),
\end{align*}
which also implies \(\sigma_{\min}(\mathbb{O}_{a, a'}) \ge \alpha^2\), holding under Assumption~\ref{ass:weaklyRev}.

We define now a block diagonal matrix \(\mathbb{B}\) created by aligning matrices \(\{\mathbb{O}_{a,a'}\}_{(a,a') \in \mathcal{A}^2}\) along the diagonal of \(\mathbb{B}\). This matrix has dimension \(A^2O^2 \times A^2S^2\) and can be directly derived from the knowledge of the emission matrices \(\mathbb{O}\).

The definition of this new quantity allows us to rewrite Equation~\eqref{eq:relationOS} in matrix notation as follows:
\begin{align}\label{eq:matRelationOS}
	\bm{d}_{A^2O^2}^{\pi} = \mathbb{B} \: \bm{d}_{A^2S^2}^{\pi},
\end{align}
where we recall that \(\bm{d}_{A^2O^2}^{\pi}\) and \(\bm{d}_{A^2S^2}^{\pi}\) are vectors of dimension \(A^2O^2\) and \(A^2S^2\), respectively.

The stationary distribution \(\bm{d}_{A^2O^2}^{\pi}\) can be directly estimated from data by counting the number of occurrences of each tuple of the type \((a_t, a_{t+1}, o_t, o_{t+1})\) collected under policy \(\pi\). Exploiting the knowledge of the block diagonal matrix \(\mathbb{B}\) and inverting Equation~\eqref{eq:matRelationOS}, we are able to derive an estimation of \(\bm{d}_{A^2S^2}^{\pi}\):
\begin{align}\label{eq:matRelationOSInv}
	\widehat{\bm{d}}_{A^2S^2}^{\pi} = \mathbb{B}^\dagger \widehat{\bm{d}}_{A^2O^2}^{\pi},
\end{align}
where \(\mathbb{B}^\dagger\) denotes the Moore-Penrose of the block matrix \(\mathbb{B}\). The solution to this equation is always defined as, by construction, it holds that \(\sigma_{\min}(\mathbb{B}) \ge \alpha^2\).

From \(\widehat{\bm{d}}_{A^2S^2}^{\pi}\), we can directly derive an estimation of \(\mathbb{T}\) by following the steps highlighted in the previous subsection. The overall procedure is detailed in Algorithm~\ref{alg:estimationAlgorithm}.

Concerning the computational complexity required for computing the inverse of matrix \(\mathbb{B}\), we notice that, since \(\mathbb{B}\) is block diagonal, the estimation procedure for each pair of arms \(a,a' \in \mathcal{A}\) can be carried out independently. As a consequence, Equation~\eqref{eq:matRelationOSInv} can be decomposed in smaller problems, only comprising the inversion of matrices \(\mathbb{O}_{a,a'}\).\footnote{From the properties of the Kronecker product, it holds that \((A \otimes B)^{-1} = A^{-1} \otimes B^{-1}\) for any two matrices \(A,B\), thus making the computational complexity even lower.}
The following lemma shows the error guarantees of the presented estimation procedure:
\begin{restatable}[]{lemma}{originalEstimationErrorBound}\label{lemma:estimationErrorBound}
    Let us assume that a policy \(\pi \in \mathcal{P}\) induces a stationary action-state distribution \(\bm{d}^{\pi}_{AS}\), such that \(\bm{d}^{\pi}_{AS}(a,s)>0 \; \forall a \in \mathcal{A}, s \in \mathcal{S}\) when interacting with a POMDP instance \(\mathcal{Q}\). If Assumption~\ref{ass:weaklyRev} holds and the process starts from its stationary distribution \(\bm{d}_{AS}^{\pi}\), then, by using Algorithm~\ref{alg:estimationAlgorithm}, with probability at least \(1 - \delta\), it holds:
	\begin{equation*}
		\|\mathbb{T} - \widehat{\mathbb{T}}\|_{F} \le \frac{2}{\alpha^2 d_{\min}^{\pi} \iota } 	\sqrt{\Big( \frac{1+\lambda_{\max}}{1-\lambda_{\max}}\Big)\frac{SA(2+ 5\log(1/\delta))}{n}}.
	\end{equation*}
\end{restatable}


\begin{algorithm}[tb]
	\caption{OAS Estimation Algorithm}
	\label{alg:estimationAlgorithm}
	\begin{algorithmic}
		\STATE {\textbf{Input:}} Observation matrix \(\{\mathbb{O}_a\}_{a \in \mathcal{A}}\), sample size \(n\), samples \(\{(a_1,o_1),(a_2,o_2), \dots, (a_n,o_n)\}\) collected using \(\pi\)
		\STATE {\textbf{Initialize:}} vector \(\bm{N}\) having dimension \(A^2O^2\), such that \(N(a,a',o,o') = 0\) for all \(a,a' \in \mathcal{A}\) and \(o,o' \in \mathcal{O}\)
		\STATE t = 1
		\WHILE{\(t \le n-1\)}
		\STATE Set \(N(a_t,a_{t+1},o_t,o_{t+1}) \: += \:1\)
		\STATE \(t=t+2\)
		\ENDWHILE
		\STATE Compute vector \(\widehat{\bm{d}}_{A^2O^2}^{\pi} = \frac{\bm{N}}{n/2}\) of consecutive action-observation distribution
		\STATE Create block matrix \(\mathbb{B}\) using Equation~\eqref{eq:fromSingleObsToDoubleObs} from \(\{\mathbb{O}_a\}_{a \in \mathcal{A}}\) 
		\STATE Compute \(\widehat{\bm{d}}_{A^2S^2}^{\pi}\) using Equation~\eqref{eq:matRelationOSInv}
		\STATE Compute \(\widehat{\mathbb{T}}\) using Equations~\eqref{eq:fromAASStoASS} and~\eqref{eq:fromASStoT}
	\end{algorithmic}
\end{algorithm}

Here, \(n\) denotes the number of time steps, while the number of consecutive non-overlapping tuples \((a,a',o,o')\) observed in the trajectory is \(n/2\). The other terms appearing in the bound are detailed in Remark~\ref{lemma:firstRemark}.
We highlight that this bound only requires ergodicity of the induced Markov chain and does not require Assumption~\ref{ass:minElem} to hold. For a formal derivation of Lemma~\ref{lemma:estimationErrorBound}, we refer the reader to Appendix~\ref{appendix:lemmaProof}.

The result shows the consistency of the presented estimation procedure, unlike general Expectation-Maximization techniques~\cite{moon1996EM}, which typically provide biased estimates and get stuck in local optima.

Lemma~\ref{lemma:estimationErrorBound} assumes that the process starts from its stationary distribution, which is quite an unrealistic case: in Appendix~\ref{appendix:lemmaProof}, we show a different version of this lemma that assumes that the chain starts from an arbitrary distribution. This result will be used in the analysis of the OAS-UCRL algorithm, which will be presented in the next section.\\

\begin{restatable}[\textbf{Dependency on POMDP Parameters}]{remark}{firstRemarkLemma}\label{lemma:firstRemark}
    The different terms appearing in the bound depend on the POMDP structure and on the policy set \(\mathcal{P}\).\\ 
	The term \(\lambda_{\max}\) represents an upper bound to the modulus of the maximum second largest eigenvalue of the sequence of inhomogeneous Markov chains induced by policy \(\pi\). We show that this term is always upper bounded by a contraction coefficient, namely the Dobrushin coefficient (see Appendix~\ref{appendix:dobrushin} for details), which is strictly smaller than \(1\) for ergodic chains. The presence of \(\lambda_{\max}\) derives from concentration results of quantities estimated from samples coming by Markov chains (see Lemma~\ref{lemma:distributionBoundStat} for details). The value \(\alpha^2\) derives from the minimum singular value of block matrix \(\mathbb{B}\); it implicitly quantifies the amount of information provided by the observations about the underlying latent state. Finally, the dependence on the terms \(d_{\min}^{\pi}:=\min_{s \in \mathcal{S}}d_S^{\pi}(s)\) derives from the transition model while \(\iota\) is the minimum action probability and characterizes the employed policy class \(\mathcal{P}\).\\ We would like to highlight that there is no dependency on the number of observations \(O\) in the bound. This result is indeed obtained using Lemma~\ref{lemma:distributionBoundStat} bounding the error in the estimate of the discrete distribution \(\bm{d}_{A^2O^2}^{\pi}\). 
    In Appendix~\ref{appendix:continuousObsSpace}, we consider how to tackle this problem in the case of a continuous observation space: a discretization procedure is proposed, which leads back to the finite observation case.\\
\end{restatable}


\begin{algorithm}[tb]
	\caption{The OAS-UCRL Algorithm}
	\label{alg:oasAlgorithm}
	\begin{algorithmic}
		\STATE {\textbf{Input:}} Observation matrix \(\{\mathbb{O}_a\}_{a \in \mathcal{A}}\), confidence level \(\delta\), base trajectory length \(T_0\)
		\STATE {\textbf{Initialize:}} \(t = 1\), \(k = 0\), policy \(\pi^{(0)}\) uniform on actions \(\mathcal{A}\), uniform belief \(b\) over states \(\mathcal{S}\), collected samples \(\mathcal{D} = \emptyset\)
		\WHILE{\(t \le T\)}
		\IF{\(k>0\)}
		\STATE Compute estimated POMDP \(\mathcal{Q}^{(k)}\) from \(\mathcal{D}\) using Algorithm~\ref{alg:estimationAlgorithm}
		\STATE Set \(\delta_k = \delta/k^3\)
		\STATE Define the set of admissible POMDPs \(\mathcal{C}_k(\delta_k)\) such that \(P(\mathcal{Q} \in \mathcal{C}_k(\delta_k)) \ge 1 - \delta_k\)
		\STATE Compute policy \(\pi^{(k)}\) and \(\mathcal{Q}^{(k)}\) maximizing Equation~\eqref{eq:ourBellmanEq} over set \(\mathcal{C}_k(\delta_k)\)
		\STATE Recompute \(b_t\) using \(\mathcal{Q}^{(k)}\) from the history of collected samples
		\ENDIF
		\STATE Set \(t_k = t\), \(\mathcal{D} = \emptyset\)
		\STATE Set \(N_k = T_02^k\)
		\FOR{\(t=t_k\) \textbf{to} \(t_k+N_k-1\)}
		\STATE Execute \(a_t \sim \pi^{(k)}(\cdot|b_t)\)
		\STATE Receive observation \(o_t\), get reward \(r_t = r(o_t)\)
		\STATE Update belief to \(b_{t+1}\) using Equation~\eqref{eq:beliefUpdate}
		\STATE Add tuple \((a_t, o_t)\) to \(\mathcal{D}\)
		\ENDFOR
		\STATE Set \(k = k+1\)
		\ENDWHILE
	\end{algorithmic}
\end{algorithm}

\begin{figure*}[t]
\begin{minipage}{.47\textwidth}
  \includegraphics[scale=1]{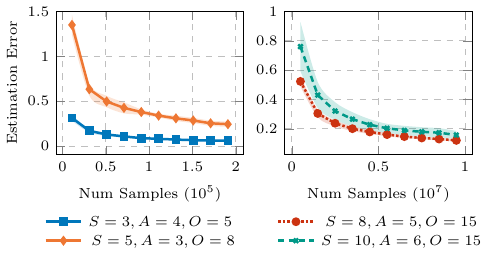}
  \caption{Frobenious norm of the estimation error of the transition model under different problem instances (10 runs, 95 \%c.i.).\\}
  \label{fig:estimationError}
\end{minipage}%
\hspace{0.9cm} 
\begin{minipage}{.47\textwidth}
  \includegraphics[width=\textwidth]{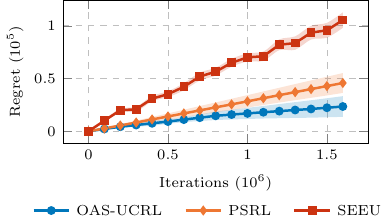}
  \caption{Regret of the different algorithms on a POMDP instance with $S=3$, $A=3$, $O=3$ (10 runs, 95 \%c.i.).}
  \label{fig:regret}
\end{minipage}
\end{figure*}

\section{OAS-UCRL Algorithm}\label{sec:OASUCRLAlgorithm}
In this section, we present the Observation-Aware Spectral UCRL (OAS-UCRL) algorithm. It is obtained by combining Algorithm~\ref{alg:estimationAlgorithm} with a UCRL-like algorithm~\cite{jaksch2010near}, designed to optimally balance the exploration-exploitation trade-off.

During the first episode \(k=0\), a uniform policy over the actions is arbitrarily chosen to collect samples for the first estimation. At the beginning of each episode \(k>0\), a new POMDP \(\mathcal{Q}^{(k)}\) is computed using Algorithm~\ref{alg:estimationAlgorithm} from samples collected in episode \(k-1\). From this estimated quantity, the algorithm builds a confidence region \(\mathcal{C}_k(\delta_k)\) with a confidence level \(1 - \delta_k\) where \(\delta_k \coloneqq \delta/k^3\). The confidence region contains the true POMDP instance \(\mathcal{Q}\) in high probability, namely \(P(\mathcal{Q} \in \mathcal{C}_k(\delta_k)) \ge 1 - \delta_k\). As previously specified, we assume the existence of an oracle computing the optimal policy corresponding to the most optimistic POMDP in \(\mathcal{C}_k(\delta_k)\), as prescribed by the optimistic principle also used in UCRL. More formally, the oracle computes:
\begin{align*}
	\pi^{(k)} = \arg \underset{\pi \in \mathcal{P}}{\max} \;\; \underset{\widetilde{\mathcal{Q}} \in \mathcal{C}_k(\delta_k)}{\max}\;\; \rho(\pi, \widetilde{\mathcal{Q}}),
\end{align*}
where we used \(\rho(\pi, \widetilde{\mathcal{Q}})\) to clarify the dependence of the average reward \(\rho\) from the employed policy \(\pi\) and the considered POMDP instance \(\widetilde{\mathcal{Q}}\), as evidenced in Equation~\eqref{eq:ourBellmanEq}.

Each time a new POMDP is estimated, the belief vector is updated starting from the first sample up to the current time step\footnote{Since the induced Markov chain is ergodic and the belief converges geometrically fast, it would also be sufficient to use a subset of the last samples for the recomputation of the belief.}. New samples are then collected using the new policy \(\pi^{(k)}\), while the belief is updated using the estimated \(\widehat{\mathbb{T}}_k\).

Since from an episode to the successive one, the employed policy is updated, the initial and stationary state distribution will not coincide, thus preventing us from directly using the results of Lemma~\ref{lemma:estimationErrorBound}. However, we report in Lemma~\ref{lemma:extendedEstimationErrorBound} (see the appendix) an extension of Lemma~\ref{lemma:estimationErrorBound}, which considers the case of a chain not starting from its stationary distribution and that can be applied in the context of the OAS-UCRL algorithm.

The OAS-UCRL algorithm proceeds in episodes of increasing length by doubling the number of samples collected at each episode. Dy denoting with \(N_0=T_0\) the length of the first episode, the subsequent episodes have length \(N_k = T_02^k\). Unfortunately, the OAS estimation procedure is not able to reuse past samples since the estimated values depend on the distribution \(\bm{d}_{AS}^{\pi^{(k)}}\) induced by the policy, which changes across episodes. The pseudocode of the approach is reported in Algorithm~\ref{alg:oasAlgorithm}.

Having described the OAS-UCRL algorithm, we are now ready to state its theoretical guarantees. We compare against an oracle knowing both the observation and transition model and using a belief-based policy \(\pi \in \mathcal{P}\). 
\begin{restatable}[]{theorem}{algorithmBound}\label{theorem:algorithmBound}
	Let us assume to have a POMDP instance \(\mathcal{Q}\) satisfying Assumption~\ref{ass:minElem}. Suppose that Assumptions~\ref{ass:weaklyRev} and~\ref{ass:policySet} hold. If the OAS-UCRL algorithm is run for \(T\) steps, with probability at least \(1 - \frac{5}{2}\delta\), it suffers from a total regret:
	\begin{equation*}
		\mathcal{R}_T \le C_1 \frac{S \left(2 + D\right)}{\alpha^2 (\epsilon \iota)^{3/2}} \sqrt{T \log(T/\delta)} + C_2,
	\end{equation*}
\end{restatable}
where \(C_1\) and \(C_2\) are constants used to report the bound in a more compact form, while \(D\) represents the diameter of the obtained belief MDP, as defined in Proposition~\ref{prop:uniformBoundBias}. Similarly to Lemma~\ref{lemma:estimationErrorBound}, the dependence in the bound on the number of observations does not appear. 
We refer the reader to Appendix~\ref{appendix:theoremProof} for a complete expression of the bound together with derivations of the result.

\section{Numerical Simulations}\label{sec:numericalSimulations}
In this section, we present two different sets of experiments. The first one is devoted to showing the estimation performance of the OAS Algorithm, while the second compares the OAS-UCRL approach in terms of regret with other baseline algorithms. More details on both types of experiments are reported in Appendix~\ref{appendix:simulationDetails}. 

\paragraph{OAS Estimation Error.} In this experiment, a belief-based policy \(\pi \in \mathcal{P}\) interacts with a POMDP instance and collects samples of type \((a, o)\) that are then provided as input to the OAS estimation procedure.  We provide the estimation error in terms of the Frobenious norm of different transition models at increasing checkpoints, as shown in Figure~\ref{fig:estimationError}. Along the x-axis, we report the number of samples used for the estimation (the order of the number of samples is reported between parenthesis). We conduct this experiment with instances having a different number of states, actions, and observations, as detailed in the legend of the Figure. As expected, the error of the estimated transition model keeps decreasing when the number of collected samples increases. We can see that the presented procedure makes good estimates even with a limited number of samples, as can be observed from the drop in the error after the first checkpoints. We highlight that comparable estimation performances with Spectral Decomposition (SD) techniques can only be obtained by requiring a much larger set of samples. For a theoretical comparison of the OAS algorithm with SD techniques, we refer to Appendix~\ref{appendix:comparison}.

\paragraph{Regret Comparison.} The second set of experiments shows the result in terms of regret of the OAS-UCRL algorithm when compared with the SEEU algorithm from~\citet{xiong2022sublinear} and the PSRL-POMDP algorithm from~\citet{jahromi2022online}. We recall that, like the OAS-UCRL algorithm, the PSRL-POMDP works under the assumption of knowing the observation model. Differently, the SEEU algorithm estimates both the transition and the observation model. To make a more fair comparison, we modified the spectral estimation used in the SEEU algorithm to help it by using the knowledge of the observation model (refer to Appendix~\ref{appendix:simulationDetails} for details). We focus on a small-scale problem following the previous works by~\citet{Azizzadenesheli2016Reinforcement} and~\citet{xiong2022sublinear} and report the results in Figure~\ref{fig:regret}.  

Our OAS-UCRL algorithm shows a reduction in terms of regret compatible with an order of $\widetilde{\mathcal{O}}(\sqrt{T})$ outperforming both baseline approaches. Indeed, the SD technique used in the SEEU algorithm requires a much larger amount of data, even for small problem instances. This aspect translates into higher estimation errors in the model parameters, inevitably leading to a higher regret. Instead, the PSRL-POMDP algorithm enjoys better regret guarantees than the SEEU algorithm, which is also favored by the knowledge of the observation model. However, the PSRL-POMDP algorithm is, in practice, unable to provide an improved estimate of the transition model since it does not come with a consistent procedure able to estimate the model parameters. This fact can be observed by the higher regret experienced with respect to the OAS-UCRL algorithm.

\section{Conclusions and Future Works}\label{sec:conclusions}
We introduced a novel estimation procedure for the POMDP parameters under the knowledge of the observation model. To the best of our understanding, this is the first algorithm working in the Partially Observable setting that can use samples collected by a belief-based policy \(\pi\) to learn the transition parameters of the model. Then, we plugged this approach into a regret minimization algorithm (OAS-UCRL) that uses an optimistic approach to handle the exploration-exploitation trade-off.
We proved theoretical guarantees for both the OAS estimation procedure and the proposed OAS-UCRL algorithm. We demonstrated the efficiency of the OAS approach over standard SD techniques while achieving a regret \(\widetilde{O}(\sqrt{T})\) for the proposed algorithm.

Here, we highlight some interesting aspects that can be tackled in future works.

\textbf{Sample Reuse.} \, The OAS-UCRL algorithm is not able to reuse past samples but only employs those coming from the last episode since each policy induces a different distribution over the action-state pairs. It may be interesting to study under which conditions all the samples can be reused. This problem is also present in SD techniques~\cite{Azizzadenesheli2016Reinforcement}.

\textbf{Stronger Oracle.} \, Solving the first point may also be helpful when considering a stronger oracle, such as the optimal deterministic belief-based policy. It is still an open question whether it is possible to achieve a regret of order \(\widetilde{O}(\sqrt{T})\) under this oracle in the infinite-horizon setting.

\bibliography{aaai25}


\newpage

\onecolumn

\renewcommand\thesection{\Alph{section}}
\setcounter{section}{0} 

\onecolumn

\section{Comparison with Related Works}\label{appendix:comparison}
In this section, we compare our work with related works from state-of-the-art, highlighting the main differences in techniques and results. In particular, we compare with (i)~\citet{Azizzadenesheli2016Reinforcement}, (ii)~\citet{xiong2022sublinear} and (iii)~\citet{jahromi2022online}.\\

\begin{table}[h!]
\centering
\begin{tabularx}{\textwidth}{|p{3.4cm}|>{\centering\arraybackslash}p{3.1cm}|>{\centering\arraybackslash}p{3.0cm}|>{\centering\arraybackslash}p{3.1cm}|>{\columncolor{poliblue3!10}\centering\arraybackslash}p{3.0cm}|}
\hline
   \parbox[c][1.3cm]{3.5cm}{} & \textbf{\citet{Azizzadenesheli2016Reinforcement}} & \textbf{\citet{xiong2022sublinear}} & 
 \textbf{\citet{jahromi2022online}} & 
 \textbf{OAS-UCRL}\\ \cline{1-5}
\end{tabularx}

\vspace{0.1cm}

\begin{tabularx}{\textwidth}{|>{\centering\arraybackslash}p{3.4cm}|>{\centering\arraybackslash}p{3.1cm}|>{\centering\arraybackslash}p{3.0cm}|>{\centering\arraybackslash}p{3.1cm}|>{\columncolor{poliblue3!10}\centering\arraybackslash}p{3.0cm}|}
\hline
  
  \parbox[c][1.0cm]{3.5cm}{\centering
 {Knowledge of Observation Model}} & No & No & Yes & Yes \\ 

  \parbox[c][1.0cm]{3.3cm}{\centering
 {Ergodicity of Induced Chain}} & Yes & Yes & Yes & Yes \\  

  \parbox[c][1.0cm]{3.3cm}{\centering
  {Minimum Transition Probability}} & No & Yes & No & Yes \\ 

  \parbox[c][1.0cm]{3.3cm}{\centering
 {Invertible Transition Model}} & Yes & Yes & No & No \\  

  \parbox[c][1.0cm]{3.3cm}{\centering
  {Full-rank Observation Model*}} & Yes & Yes & Yes** & Yes \\ 

  \parbox[c][1.0cm]{3.3cm}{\centering
  {Minimum Action Probability}} & Yes & No & No & Yes \\ 

  \parbox[c][1.0cm]{3.3cm}{\centering
 {Consistent Transition Model Estimation}} & No & No & Yes & No \\  

  \parbox[c][1.0cm]{3.3cm}{\centering
 {Consistent Belief Estimation}} & No & No & Yes & No \\  \cline{1-5}

\end{tabularx}

\vspace{0.1cm}

\begin{tabularx}{\textwidth}{|>{\centering\arraybackslash}p{3.4cm}|>{\centering\arraybackslash}p{3.1cm}|>{\centering\arraybackslash}p{3.0cm}|>{\centering\arraybackslash}p{3.1cm}|>{\columncolor{poliblue3!10}\centering\arraybackslash}p{3.0cm}|}
\hline
  
 \parbox[c][1.2cm]{3.5cm}{\centering
 {\textbf{Estimation Technique}}} & SD & SD & Bayesian Update & OAS Algorithm\\  \cline{1-5}
     
 \parbox[c][1.2cm]{3.5cm}{\centering
 {\textbf{Consistent Estimation}}} & \Large{\cmark} & \Large{\cmark} & \Large{\xmark} & \Large{\cmark} \\ \cline{1-5}

 \parbox[c][1.0cm]{3.5cm}{\centering
 {Works under Uniform Policies}} & \Large{\cmark} & \Large{\cmark} & \Large{\xmark} & \Large{\cmark} \\ 

 \parbox[c][1.0cm]{3.5cm}{\centering
 {Works under Memoryless Policies}} & \Large{\cmark} & \Large{\xmark} & \Large{\xmark} & \Large{\cmark} \\ 

 \parbox[c][1.0cm]{3.5cm}{\centering
 {Works under Belief-based Policies}} & \Large{\xmark} & \Large{\xmark} & \Large{\xmark} & \Large{\cmark} \\ \cline{1-5}
\end{tabularx}

\vspace{0.1cm}

\begin{tabularx}{\textwidth}{|>{\centering\arraybackslash}p{3.4cm}|>{\centering\arraybackslash}p{3.1cm}|>{\centering\arraybackslash}p{3.0cm}|>{\centering\arraybackslash}p{3.1cm}|>{\columncolor{poliblue3!10}\centering\arraybackslash}p{3.0cm}|}
\hline
    
  \parbox[c][1.2cm]{3.5cm}{\centering
 {\textbf{Oracle}}} & Optimal Stochastic Memoryless Policy  & Optimal POMDP policy & Optimal POMDP policy & Optimal Stochastic POMDP policy \\   \cline{1-5}
  
  \parbox[c][1.2cm]{3.5cm}{\centering
 {\textbf{Algorithm Type}}} & Optimistic & Alternating Explor. - Exploit. phases & TS-based & Optimistic \\ \cline{1-5}

  \parbox[c][1.2cm]{3.5cm}{\centering
 {\textbf{Regret}}} & $\widetilde{\mathcal{O}}(\sqrt{T})$ & $\widetilde{\mathcal{O}}(T^{2/3})$ & $\mathcal{O}(T^{2/3})$*** & $\widetilde{\mathcal{O}}(\sqrt{T})$ \\  \cline{1-5}
\end{tabularx}
\caption{Tables comparing our work with most relevant related works. The following abbreviations are used: SD (Spectral Decomposition), and TS (Thompson Sampling).\\
(*): The full-rank assumption directly maps to the weakly-revealing assumption (Assumption~\ref{ass:weaklyRev});\\
(**): The used property can be related to the full-rank assumption. It states that, for any two different transition models, a different distribution over action-observation sequences is induced;\\(***): Bayesian regret.
}
\label{tab:comparisonRL}
\end{table}

\paragraph{\citet{Azizzadenesheli2016Reinforcement}} employ \emph{Spectral Decomposition} (SD) techniques to learn both the observation and the reward model. They extend the standard SD technique to work with samples coming from a memoryless policy, which is different from standard approaches that commonly employ uniform policies for sample collection. Their assumptions are mainly related to those used in SD approaches and, as in our case, they further require an assumption on the minimum action probability. In addition, they do not require the assumption of the minimum value of the transition probability (which is necessary for our setting to bound the error in the estimated belief vector) since the set of memoryless policies does not require keeping the computation of a belief over the states.\\
They obtain a regret result of \(\widetilde{\mathcal{O}}(\sqrt{T})\) when comparing against the Optimal Stochastic Memoryless policy, which is a weaker oracle when compared to the one we employ, which instead uses stochastic belief-based policies.
\paragraph{\citet{xiong2022sublinear}} proves a  \(\widetilde{\mathcal{O}}(T^{2/3})\) regret guarantee under the Optimal POMDP policy using standard SD techniques for parameter estimation (on samples collected under uniform policies) and a phased algorithm that alternates between purely explorative and exploitative phases. They employ common assumptions for SD techniques and, similarly to our setting, they have the one-step reachability assumption, which is needed to bound the belief error with the error of the estimated transition model.
\paragraph{\citet{jahromi2022online}} present a work in which the observation model is assumed to be known. They employ a Bayesian approach for jointly estimating the model and efficiently exploiting the available information. They reach a Bayesian regret of order \(\mathcal{O}(T^{2/3})\) against the optimal POMDP policy. However, they are able to reach this result under two core assumptions:
\begin{itemize}
    \item they assume that their estimator is consistent (see their Assumption 4). This assumption in the partially observable setting is quite strong since consistency has only been proved for Spectral Decomposition techniques. The consistency of Bayesian estimators has been only shown to hold under the fully observable and also under the finite-parameter case. Differently from their work, we devise the new OAS estimation procedure and show how this approach leads to consistent estimates.
    \item they assume that the error of the estimated belief scales with order \(\mathcal{O}(n^{-1/2})\) with \(n\) the number of collected samples (see their Assumption 3). We obtain a similar convenient scaling by assuming the minimum value for the transition probability and then using Proposition~\ref{prop:beliefErrorPOMDP}.\\
\end{itemize}

\paragraph{Table Comparison}
Table~\ref{tab:comparisonRL} summarizes the main differences between our work with those just mentioned.
They are mainly compared in terms of: 
\begin{itemize}
    \item Assumptions (first sub-table);
    \item Estimation Techniques and its properties (second sub-table);
    \item Algorithm Properties (third sub-table).\\
\end{itemize}

\subsection{Comparison of OAS with Spectral Decomposition Approaches}\label{lemma:OAScomparisonSD}
In this section, we highlight the main differences between our approach and the SD procedure employed in~\citet{Azizzadenesheli2016Reinforcement}.\\

We recall that SD techniques have been commonly used to learn the transition and the observation model in HMM jointly, and their use has also been extended to learning POMDP models. To begin with, SD techniques require invertibility of the transition matrix \(\mathbb{T}_a\) associated with each action, while this is not required for our approach.\\ 
By taking into account the dependency on the problem parameters, their estimation error scales with an overall dependency of \(S^{3/2}\) with respect to states, compared against the dependence of order \(S^{1/2}\) experienced by our approach, as shown in Lemma~\ref{lemma:estimationErrorBound}. The estimation error of the transition matrices in~\citet{Azizzadenesheli2016Reinforcement} also depends on terms that cannot be directly mapped into our setting. However, by only considering the common terms and by comparing them directly, we can see that they scale with a dependence of \(1/(\alpha^4\iota^2d_{\min}^{3/2})\), while Lemma~\ref{lemma:estimationErrorBound} shows a dependency of \(1/(\alpha^2\iota d_{\min})\) that improves over all the considered terms.\\

The last aspect we would like to highlight is that SD techniques are generally used under fully explorative policies, such as uniform policies or round-robin approaches~\cite{zhou2021regime, xiong2022sublinear, hsu2012spectral}.~\citet{Azizzadenesheli2016Reinforcement} managed to adapt their use to samples collected from memoryless policies. However, to the best of our knowledge, there are no estimation procedures based on SD techniques that use samples collected from belief-based policies to estimate the POMDP parameters. Our OAS procedure is, instead, the first that can be used under any policy \(\pi\) that induces an ergodic Markov chain on the POMDP. 

\vspace{0.2cm}

\section{Notation}
We will define some notations here that will be used in the following.\\ 
We will use notation \([N]\) to denote an interval from \(1\) to \(N\).\\
We will use the bold letter \(\bm{d}^{\pi}\) to denote a generic vector of stationary distribution induced by policy \(\pi \in \mathcal{P}\). We will also use the subscript to denote the length of the vector. For example, we will use the quantity \(\bm{d}^{\pi}_{S} \in \Delta(\mathcal{S})\) to denote the vector of length \(S\) such that \(\bm{d}^{\pi}_{S}(s)= P(S=s|\pi)\) and analogously we will denote with \(\bm{d}^{\pi}_{AS} \in \Delta(\mathcal{A}\times \mathcal{S})\) such that \(\bm{d}^{\pi}_{AS}(a,s) = P(A=a, S=s|\pi)\). In a similar way, we define the quantities:
\begin{itemize}
    \item \(\bm{d}^{\pi}_{A^2O^2} \in \Delta(\mathcal{A}^2\times \mathcal{O}^2)\) to denote the stationary distribution over the tuples \((a,a',o,o')\),
    \item \(\bm{d}^{\pi}_{A^2S^2} \in \Delta(\mathcal{A}^2\times \mathcal{S}^2)\) to denote the stationary distribution over the tuples \((a,a',s,s')\),
    \item \(\bm{d}^{\pi}_{AO^2} \in \Delta(\mathcal{A}\times \mathcal{O}^2)\) to denote the stationary distribution over the tuples \((a,o,o')\),
    \item \(\bm{d}^{\pi}_{AS^2} \in \Delta(\mathcal{A}\times \mathcal{S}^2)\) to denote the stationary distribution over the tuples \((a,s,s')\),
\end{itemize}
and so on.\\ 
Let us recall the temporal relation between the elements in the tuple and let us consider for example the elements of the vectors \(\bm{d}^{\pi}_{A^2O^2}\) and \(\bm{d}^{\pi}_{AO^2}\). We denote with:
\begin{itemize}
    \item \(d_t^{\pi}(a,a',o,o') = P(A_{t}=a, A_{t+1}=a', O_t=o, O_{t+1}=o'|\pi)\) and we denote the limit version with \(\bm{d}^{\pi}_{A^2O^2}(a,a',o,o') = \lim_{t \to \infty}d_t^{\pi}(a,a',o,o')\),
    \item  \(d_t^{\pi}(a,o,o') = P(A_{t}=a, O_t=o, O_{t+1}=o'|\pi)\) and we denote the limit version with \(\bm{d}^{\pi}_{AO^2}(a,o,o') = \lim_{t \to \infty}d_t^{\pi}(a,o,o')\).
\end{itemize}
Thus, the terms with the apex will refer to a successive timestamp (\(t+1\)) with respect to those without the apex (\(t\)).

\vspace{0.2cm}

\section{Proof of Proposition~\ref{proposition:uniqueStationaryPOMDP}}\label{appendix:propUniqueStat}
In this section, we will present the proof of Proposition~\ref{proposition:uniqueStationaryPOMDP}. We report here for clarity the statement of the proposition.
\uniqueStationaryPOMDP*
\begin{proof}
    The objective of the proof is to show that given the condition on the ergodicity (the induced chain is thus irreducible and aperiodic) there exists a limiting action-state distribution \(\bm{d}^{\pi}_{AS} \in \Delta(\mathcal{A} \times \mathcal{S})\) where the apex \(\pi\) denotes that the distribution is induced by the belief-based policy \(\pi \in \mathcal{P}\). From the existence of this distribution, we can easily extend it to the existence of a limiting distribution on consecutive action-state pairs \(\bm{d}^{\pi}_{A^2S^2} \in \Delta(\mathcal{A}^2 \times \mathcal{S}^2)\).\\
    When all states are reachable under the used policy \(\pi \in \mathcal{P}\), the following Bellman Expectation Equation holds:
    \begin{align*}
	\rho^{\pi} + v_{\pi}(b)\; = \underset{a \sim \pi(\cdot|b)}{\mathbb{E}}\left[g(b,a) + \int_{\mathcal{B}}P(\,db'|b, a) v_{\pi}(b') \right].
    \end{align*}
    We will now define a quantity useful for the following analysis. We denote the expected reward of an action \(a \in \mathcal{A}\) while being in state \(s \in \mathcal{S}\) as:
    \begin{align}\label{def:stateActionRew}
        \mu(s, a) = \sum_{o \in \mathcal{O}}r(o) \mathbb{O}_{a}(o|s) = \bm{r}^\top \mathbb{O}_{a}(\cdot|s),
    \end{align}
    where the last term represents the vector form of the expression and \(\bm{r}\) is a vector of size \(O\) containing the deterministic reward associated with each observation.
    By using the Bellman Expectation Equation and quantity \(\mu(s, a)\) just defined, we can also express the average reward induced by policy \(\pi\) as:
    \begin{align*}
        \rho^{\pi} = \sum_{s \in \mathcal{S}} \sum_{a \in \mathcal{A}} d^{\pi}(a, s)  \mu(s,a).
    \end{align*}
    Let us now consider a specific \(a' \in \mathcal{A}\) and a state \(s' \in \mathcal{S}\) and let us define the reward function as \(\mu(s, a) = \mathds{1}[s=s', a=a']\). 
    By applying this reward function, we get that \(\rho^{\pi} = d^{\pi}(a',s')\), which ensures as well the existence of \(d^{\pi}(a',s')\) since \(\rho^{\pi}\) is ensured to exists from the Bellman Equation. By extending this consideration to all other pairs of actions and states and choosing the reward function accordingly, we can show a unique stationary distribution \(\bm{d}^{\pi}_{AS}\) on action-state pairs exists. From this result, the existence of a stationary distribution \(\bm{d}^{\pi}_{A^2S^2}\) on consecutive action-state pairs directly follows, thus concluding the proof.
\end{proof}

\section{Proof of main Lemmas and Theorems}
In this section, we provide proofs and derivations for our theoretical claims. More specifically, we will provide the derivations for the bound appearing in Lemma~\ref{lemma:estimationErrorBound}. We will present a different version of this result (Lemma~\ref{lemma:extendedEstimationErrorBound}) that is more tailored to our setting since it uses Assumption~\ref{ass:minElem} and does assume to start from the stationary distribution. Finally, we will show the derivations for Theorem~\ref{theorem:algorithmBound}.


\subsection{Proof of Lemma~\ref{lemma:estimationErrorBound}}\label{appendix:lemmaProof}
Before presenting the proof of Lemma~\ref{lemma:estimationErrorBound}, we report here the statement.

\originalEstimationErrorBound*
\begin{proof}
	Let us first recall from the notation that we denote with \(\mathbb{T}\) the transition matrix having size \(SA \times S\) where each row \(\mathbb{T}({\cdot|s,a}) \in \Delta(\mathcal{S})\) denotes the probability distribution of reaching the next state, given the current state \(s\) and the chosen action \(a\). While we denote with \(\mathbb{T}_a\) the transition matrix having size \(S \times S\) where the probabilities are conditioned on having taken action \(a\). Furthermore, we recall that \(\bm{d}^{\pi}_{A^2S^2} \in \Delta(\mathcal{A}^2 \times \mathcal{S}^2)\) denotes the vector of size \(A^2S^2\) representing the stationary distribution over the tuples \((a,a',s,s') \in \mathcal{A}^2 \times \mathcal{S}^2\), while \(\bm{d}^{\pi}_{A^2O^2} \in \Delta(\mathcal{A}^2 \times \mathcal{O}^2)\) denotes the vector of size \(A^2O^2\) representing the stationary distribution over the tuples \((a,a',o,o') \in \mathcal{A}^2 \times \mathcal{O}^2\). We define their estimated quantities as \(\widehat{\bm{d}}^{\pi}_{A^2S^2}\) and \(\widehat{\bm{d}}^{\pi}_{A^2O^2}\) respectively.\\ From now on we will omit apex \(\pi\) when referring to the vectors of induced stationary distribution, thus having \(\bm{d}_{A^2S^2}\) and \(\bm{d}_{A^2O^2}\).\\
	The proof of the presented bound can be broken down into two main parts. The former part bounds the estimation error of the vector representing the stationary distribution of consecutive action-state pairs \(\bm{d}_{A^2S^2}\), while the latter part bounds the estimation error of the transition matrix \(\mathbb{T}\) with respect to the estimation error of \(\bm{d}_{A^2S^2}\).\\
 
 We will start by proving the first result. We have that:
	\begin{align}
		\|\bm{d}_{A^2S^2} - \widehat{\bm{d}}_{A^2S^2}\|_2 = & \lVert \mathbb{B}^\dagger (\bm{d}_{A^2O^2} - \widehat{\bm{d}}_{A^2O^2})\rVert_2 \notag \\
		\le & \|\mathbb{B}^\dagger\|_2 \lVert \bm{d}_{A^2O^2} - \widehat{\bm{d}}_{A^2O^2}\rVert_2 \notag \\
		= & \frac{1}{\sigma_{\min}(\mathbb{B})} \lVert \bm{d}_{A^2O^2} - \widehat{\bm{d}}_{A^2O^2}\rVert_2\\
		\le & \frac{1}{\alpha^2} \lVert \bm{d}_{A^2O^2} - \widehat{\bm{d}}_{A^2O^2}\rVert_2, \label{lp:2normDif}
	\end{align}
	where the first equality follows from Equation~\eqref{eq:matRelationOSInv}, while the first inequality follows by the consistency property of matrices: we bound the 2-norm on the first line with the product of the 2-norm of the vector of  differences and the spectral norm of matrix \(\mathbb{B}^\dagger\).\\
	The second inequality derives by definition of the block matrix \(\mathbb{B}\), that is composed by submatrices \(\{\mathbb{O}_{a,a'}\}_{(a,a') \in \mathcal{A}^2}\) for which it holds that \(\sigma_{\min}(\mathbb{O}_{a,a'}) \ge \alpha^2\) for all \((a,a') \in \mathcal{A}^2\). For the properties of block diagonal matrices, it also follows that \(\sigma_{\min}(\mathbb{B}) \ge \alpha^2\).\\
 
	We will now focus on bounding the last term appearing in~\ref{lp:2normDif} representing the estimation error of the stationary distribution of consecutive action-observation pairs \(\bm{d}_{A^2O^2}\). We recall that this vector represents the parameters of a categorical distribution on the consecutive action-observation couples. If the samples used to estimate this distribution were independent, we could easily apply the standard version of McDiarmid's inequality~\cite{mcdiarmid1989method} for categorical distributions.\\ However, in our case, we have a dependency among samples since they originated from the policy \(\pi\) interacting with the environment. 
	Indeed, the tuples \((a,a',o,o')\) observed every 2 steps are realizations of stochastic functions that depend on the underlying tuple \((a,a',s,s')\). The sequence of these action-state tuples, in turn, depends on the Markov chains induced by policy \(\pi\).\\
	
	Let us now see how these Markov chains are characterized. 
	For any time \(t\), there is a Markov chain determined by a transition matrix \(\mathbb{M}_t\) with size \(SA \times SA\) and such that \(\mathbb{M}(s',a'|s,a)\) represents the joint probability of reaching \(S_{t+1}=s'\) and taking action \(A_{t+1}=a'\) under policy \(\pi\), after having taken action \(A_t=a\) in state \(S_t=s\). We have that \(\mathbb{M}(\cdot,\cdot|s,a) \in \Delta(\mathcal{S}\times \mathcal{A})\). We will refer to them as state-action transition matrices.\\ These matrices \((\mathbb{M}_t)_{t\ge1}\) change at every time instant, and this inhomogeneity is due to the non-markovian nature of the policy that updates the belief at every time step, thus continuously changing the action probabilities.
	Despite their inhomogeneity, the ergodicity assumption ensures that these chains converge to an invariant measure represented by the stationary action-state distribution \(\bm{d}_{AS}\) induced by policy \(\pi\).\\
    ~\citet{fanHoeffding2021} show that the speed of convergence of a sequence of Inhomogeneous Markov chains depends on \(\lambda_{\max}\), which is defined as:
    \begin{align*}
        \lambda_{\max}:= \underset{t\ge 1}{\max}|\lambda(\mathbb{M}_t)|,
    \end{align*}
    where \(|\lambda(\mathbb{M}_t)|\) denotes the modulus of the second largest eigenvalue of matrix \(\mathbb{M}_t\). Since, by assumption, the sequence of chains is ergodic, it holds that \(\lambda_{\max} < 1\), where the value of \(\lambda_{\max}\) depends on the model characterizing the POMDP and the properties of policy \(\pi \in \mathcal{P}\).
	Given these considerations, we have all the information to bound the quantity in Equation~\eqref{lp:2normDif}. We will do it by using the result formulated in Lemma~\ref{lemma:distributionBoundStat} that bounds the estimation error of probability distributions of samples coming from Markov chains. Using Lemma~\ref{lemma:distributionBoundStat}, with probability at least \(1 - \delta\), the following holds:
	\begin{align}\label{eq:boundAOHom}
		\lVert \bm{d}_{A^2O^2} - \widehat{\bm{d}}_{A^2O^2}\rVert_2 \le \sqrt{\Big( \frac{1+\lambda_{\max}}{1-\lambda_{\max}}\Big)\frac{2+ 5\log(1/\delta)}{n}}.
	\end{align}
 While in Lemma~\ref{lemma:distributionBoundStat} we have a dependency \(\sqrt{1/2n}\), here we only have \(\sqrt{1/n}\). The 2 term is removed since the meaning of the quantity \(n\) is different in the two lemmas: this Lemma uses \(n\) as the number of interaction time steps, while Lemma~\ref{lemma:distributionBoundStat} uses \(n\) to refer to the tuple \((a,a'o,o')\) collected in 2 consecutive time steps.\\
By combining the results just obtained, we can finally conclude the first part of the proof having that:
\begin{align}
    \|\bm{d}_{A^2S^2} - \widehat{\bm{d}}_{A^2S^2}\|_2 \le \frac{1}{\alpha^2} \sqrt{\Big( \frac{1+\lambda_{\max}}{1-\lambda_{\max}}\Big)\frac{2+ 5\log(1/\delta)}{n}}\label{lp:firstPartBound},
\end{align}
holding with probability at least \(1 - \delta\).\\

We proceed now with the second part of the proof which consists in bounding the error of the transition matrix \(\mathbb{T}\) that derives from the estimation error of vector \(\bm{d}_{A^2S^2}\). In the following, we will use variable \(i \in [AS]\) to map any couple \((a,s) \in \mathcal{A} \times \mathcal{S}\) thus denoting the rows of the transition matrix \(\mathbb{T}\). Furthermore, we will use notation \(\mathbb{T}(i,j)\) with the same meaning of \(\mathbb{T}(j|i)\), with \(j \in \mathcal{S}\) and \(i\) specified as above.\\
The proof goes on by transforming vector \(\widehat{\bm{d}}_{A^2S^2}\) into vector \(\widehat{\bm{d}}_{AS^2}\) using Equation~\eqref{eq:fromAASStoASS} that we report here for the estimated quantities:
\begin{align*}
    \widehat{d}_{AS^2}(a,s,s') = \sum_{a' \in \mathcal{A}} \widehat{d}_{A^2S^2}(a,a',s,s')\quad \forall a \in \mathcal{A}, \forall s,s' \in \mathcal{S}
\end{align*}
where vector \(\widehat{\bm{d}}_{AS^2}\) represents the estimate of the stationary distribution \(\bm{d}_{AS^2} \in \Delta(\mathcal{A} \times \mathcal{S}^2)\) of the tuple \((a,s,s')\) induced by policy \(\pi\). We can show that:
\begin{align}\label{lp:aggregationBound}
    \|\bm{d}_{AS^2} - \widehat{\bm{d}}_{AS^2}\|_2 \le \sqrt{A} \|\bm{d}_{A^2S^2} - \widehat{\bm{d}}_{A^2S^2}\|_2,
\end{align}
where this inequality is obtained by the Aggregation Lemma~\ref{lemma:aggregation} presented in Appendix~\ref{appendix:usefulLemmas}.
For what follows, we will denote quantity \(d_{AS^2}(a,s,s')\) also using notation \(d_{AS^2}(i,s')\) where \(i\) maps the couple \((a,s)\). We will denote with \(d_{AS^2}(i,\cdot)\) the vector of size \(S\) containing the quantity \(d_{AS^2}(i,s')\) for each \(s' \in \mathcal{S}\). Also, we characterize the distribution \(\bm{d}_{AS} \in \Delta(\mathcal{A} \times \mathcal{S})\) as a function of \(\bm{d}_{AS^2}\) as follows:
\begin{align}
    d_{AS}(a,s) = \sum_{s' \in \mathcal{S}}d_{AS^2}(a,s,s')\quad \forall a \in \mathcal{A}, \forall s \in \mathcal{S}.\label{def:d(a,s)}
\end{align}
As done above, we will refer to each element \(d_{AS}(a,s)\) also using notation \(d_{AS}(i)\), with index \(i\) mapping the tuple \((a,s)\).\\

We go on now with the proof. Since \(\widehat{\bm{d}}_{AS^2}\) is directly obtained from \(\widehat{\bm{d}}_{A^2S^2}\) which is itself obtained from Equation~\eqref{eq:matRelationOSInv}, it can contain negative terms. We thus need to convert the negative terms into non-negative ones by setting them to 0. We thus define this newly transformed quantity with \(\widebar{\bm{d}}_{AS^2}\) and for this quantity it is straightforward to see that:
\begin{align*}
    \|\bm{d}_{AS^2} - \widebar{\bm{d}}_{AS^2}\|_2 \le \|\bm{d}_{AS^2} - \widehat{\bm{d}}_{AS^2}\|_2,
\end{align*}
which holds since all the terms of vector \(\bm{d}_{AS^2}\) are positive.\\
Having introduced this quantity, we can now write:
\begin{align}
    \|\mathbb{T} - \widehat{\mathbb{T}}\|_{F} & = \sqrt{\sum_{i \in \mathcal{A} \times \mathcal{S}} \sum_{j \in \mathcal{S}} (\mathbb{T}(i,j) - \widehat{\mathbb{T}}(i,j))^2} = \sqrt{\sum_{i \in \mathcal{A} \times \mathcal{S}} \|\mathbb{T}(i,\cdot) - \widehat{\mathbb{T}}(i,\cdot)\|_2^2}\notag\\
    & = \sqrt{\sum_{i \in \mathcal{A} \times \mathcal{S}} \left\lVert\frac{d_{AS^2}(i,\cdot)}{\|d_{AS^2}(i,\cdot)\|_1} - \frac{\widebar{d}_{AS^2}(i, \cdot)}{\|\widebar{d}_{AS^2}(i,\cdot)\|_1} \right\rVert_2^2} \notag\\
    & \le \sqrt{\sum_{i \in \mathcal{A} \times \mathcal{S}} \left\lVert\frac{d_{AS^2}(i,\cdot)}{\|d_{AS^2}(i,\cdot)\|_2} - \frac{\widebar{d}_{AS^2}(i,\cdot)}{\|\widebar{d}_{AS^2}(i,\cdot)\|_2} \right\rVert_2^2}\label{lp:normRelations}\\
    & \le \sqrt{\sum_{i \in \mathcal{A} \times \mathcal{S}} 
        \frac{4 \lVert d_{AS^2}(i,\cdot) - \widebar{d}_{AS^2}(i,\cdot) \rVert_2^2}{\max\{\|d_{AS^2}(i,\cdot)\|_2, \|\widebar{d}_{AS^2}(i,\cdot)\|_2\}^2}} \label{lp:lemmaG}\\
    & \le \sqrt{\sum_{i \in \mathcal{A} \times \mathcal{S}} 
        \frac{4 \lVert d_{AS^2}(i,\cdot) - \widehat{d}_{AS^2}(i,\cdot) \rVert_2^2}{\|d_{AS^2}(i,\cdot)\|_2^2}} \notag\\
    & \le \sqrt{\sum_{i \in \mathcal{A} \times \mathcal{S}} 
        \frac{4 S \lVert d_{AS^2}(i,\cdot) - \widebar{d}_{AS^2}(i,\cdot) \rVert_2^2}{d_{\min}^2 \iota^2}}\label{lp:norm2toNorm1}\\
    & = \sqrt{\frac{4 S \|\bm{d}_{AS^2} -\widebar{\bm{d}}_{AS^2}\|_2^2}{d_{\min}^2 \iota^2}}\label{lp:matrixToVecBound} \\
    & \le \sqrt{\frac{4 S \|\bm{d}_{AS^2} -\widehat{\bm{d}}_{AS^2}\|_2^2}{d_{\min}^2 \iota^2}} = \frac{2 \sqrt{S} \|\bm{d}_{AS^2} -\widehat{\bm{d}}_{AS^2}\|_2}{d_{\min} \iota}.\label{lp:lastIneqBoundT}
\end{align}
where the first inequality in line~\ref{lp:normRelations} derives from the well-known relation between norms \(\|\widebar{d}_{AS^2}(i,\cdot)\|_1 \ge \|\widebar{d}_{AS^2}(i,\cdot)\|_2\), while line~\ref{lp:lemmaG} derives from Lemma~\ref{lemma:trulyBatch} (see Appendix~\ref{appendix:usefulLemmas}). Line~\ref{lp:norm2toNorm1} derives instead from the following considerations: for any vector \(d_{AS^2}(i,\cdot)\) of dimension \(S\), it holds:
\begin{align*}
    \|d_{AS^2}(i,\cdot)\|_2^2 = \sum_{j \in \mathcal{S}}d_{AS^2}(i,j)^2 \ge \frac{1}{S} \left( \sum_{j \in \mathcal{S}}d_{AS^2}(i,j)\right)^2 = \frac{1}{S} d_{AS}(i)^2 \ge \frac{1}{S} (\iota d_{\min})^2,
\end{align*}
where the first inequality follows from the fact that \(\sqrt{X}\|\bm{x}\|_2 \ge \|\bm{x}\|_1 \forall \bm{x} \in \mathbb{R}^X\). The second equality instead directly derives from the definition of \(d_{AS}(i)\) in~\ref{def:d(a,s)}. 
For the last inequality, we have used a quantity \(d_{\min}\iota := \min_{i \in \mathcal{A} \times \mathcal{S}}d_{AS}(i)\) as a lower bound to each value of the distribution \(\bm{d}_{AS}\) induced by policy \(\pi \in \mathcal{P}\). We recall here that \(\iota\) represents the minimum probability of taking each action, according to the definition of the policy class \(\mathcal{P}\) in~\eqref{def:policySet}.\\
Finally, the equality in line~\ref{lp:matrixToVecBound} is obtained by simply noting that:
\begin{align*}
    \sum_{i \in \mathcal{A} \times \mathcal{S}} \|d_{AS^2}(i,\cdot) - \widebar{d}_{AS^2}(i,\cdot)\|_2^2 = \|\bm{d}_{AS^2} -\widebar{\bm{d}}_{AS^2}\|_2^2,
\end{align*}
by the definition of \(\bm{d}_{AS^2}\) and \(\widebar{\bm{d}}_{AS^2}\) respectively.

By combining the results obtained in~\eqref{lp:firstPartBound},~\eqref{lp:aggregationBound} and~\eqref{lp:lastIneqBoundT}, we are able to obtain the final result holding with probability at least \(1 - \delta\):
\begin{align*}
    \|\mathbb{T} - \widehat{\mathbb{T}}\|_{F} \le \frac{2}{\alpha^2 d^{\pi}_{\min} \iota } \sqrt{\Big( \frac{1+\lambda_{\max}}{1-\lambda_{\max}}\Big)\frac{SA(2+ 5\log(1/\delta))}{n}}
\end{align*}
where we made explicit the dependence of \(d_{\min}\) from the employed policy \(\pi\).\\ 
This last formulation concludes the proof.
\end{proof}
\vspace{0.2cm}
\subsection{Extension of Lemma~\ref{lemma:estimationErrorBound}}\label{appendix:extensionLemma}
The objective of this subsection is to present a different version of Lemma~\ref{lemma:estimationErrorBound} that is more suitable for our setting since it does not assume that the chain starts from its stationary distribution. It will be used to prove the result of Theorem~\ref{theorem:algorithmBound}. Let us first state the new extended Lemma and then we will proceed with the proof.

\begin{restatable}[]{lemma}{extendedestimationErrorBound}\label{lemma:extendedEstimationErrorBound}
	Let us assume to have a POMDP instance \(\mathcal{Q}\) with a finite number of states \(\mathcal{S}\) and a finite number of actions \(\mathcal{A}\). Let Assumptions~\ref{ass:minElem},~\ref{ass:weaklyRev} and~\ref{ass:policySet} hold. By denoting with \(\bm{\nu} \in \Delta(\mathcal{A} \times \mathcal{S})\) the initial distribution and with \(\bm{d}_{AS} \in \Delta(\mathcal{A} \times \mathcal{S})\) the stationary distribution induced by the policy \(\pi\), by using the estimation procedure defined Algorithm~\ref{alg:estimationAlgorithm}, with probability at least \(1 - \delta\), it holds:
	\begin{equation*}
		\|\mathbb{T} - \widehat{\mathbb{T}}\|_{F} \le \frac{2}{\alpha^2 (\epsilon \iota)^{3/2} } \sqrt{\frac{2C(\bm{\nu},\bm{d}_{AS})+ 5\log(C(\bm{\nu},\bm{d}_{AS})/\delta))}{n}},
	\end{equation*}
where \(C(\bm{\nu},\bm{d}_{AS})\) is defined as \(C(\bm{\nu},\bm{d}_{AS})=\left\| \frac{\bm{\nu}}{\bm{d}_{AS}}\right\|_{\infty}\) and \(\frac{\bm{\nu}}{\bm{d}_{AS}}\) is a vector of size \(AS\) containing the element-wise ratio of the elements of the two vectors. 
\end{restatable}

\begin{proof}
	The proof of this lemma largely follows from that of Lemma~\ref{lemma:estimationErrorBound}. The main difference is that the process may not start from its stationary distribution. Furthermore, since Assumption~\ref{ass:minElem} holds, we are able to better characterize the contraction coefficient of the chain and, thus, its related absolute spectral gap. We will start by analyzing this last point.
	
	We refer to the concept of state-action transition matrices \((\mathbb{M}_t)_{t\ge1}\) defined in the proof of Lemma~\ref{lemma:estimationErrorBound}.
	Let us show that if Assumption~\ref{ass:minElem} holds, it is always true that the modulus of the second largest eigenvalue \(|\lambda|\) is strictly smaller than 1.\\
	Under Assumption~\ref{ass:minElem} and assuming that \(\pi \in \mathcal{P}\) (Assumption~\ref{ass:policySet}), it follows that we can lower bound the minimum value associated with each matrix \(\mathbb{M}_t\). In particular, by having variables \(i,j \in \mathcal{S} \times \mathcal{A}\), we have that:
	\begin{align*}
		\tau:= \underset{i,j}{\min} \; \mathbb{M}_t(i|j) \ge \epsilon\iota > 0 \qquad \forall t \ge 1,
	\end{align*}
    where \(j\) and \(i\) encode respectively the index of the row and of the column of the state-action transition matrix \(\mathbb{M}_t\).\\
    In order to better characterize the second largest eigenvalue of the induced Markov chains, we will use the notion of Dobrushin coefficient~\cite{krish2016partially} and exploit its properties. The Dobrushin coefficient is a contraction coefficient that controls the speed of convergence of the Markov chain induced by the sequence \((\mathbb{M}_t)_{t\ge1}\) towards the stationary distribution. In particular, for any matrix \(\mathbb{D}\),~\citet{krish2016partially} show that the contraction coefficient \(\rho(\mathbb{D}) \in [0,1]\) satisfies the following contraction equation:
    \begin{align*}
	\| \mathbb{D}^\top\nu - \mathbb{D}^\top\bar{\nu}\|_1 \le \rho(\mathbb{D})\|\nu - \bar{\nu}\|_1,
    \end{align*}
    with \(\nu\) and \(\bar{\nu}\) being any two initial distributions, and \(\mathbb{D}^\top\) being the transpose of matrix \(\mathbb{D}\). We defer to Appendix~\ref{appendix:dobrushin} for a thoughtful presentation of its properties.\\
    For our setting, it is relevant to know that when the elements of matrix \(\mathbb{D}\) are all positive, the associated \(\rho(\mathbb{D})\) coefficient is strictly smaller than 1. In particular, we have:
    \begin{align*}
	\rho(\mathbb{M}_t) & = 1 - \min_{i,j \in \mathcal{S}\times \mathcal{A}} \sum_{l \in \mathcal{S}\times \mathcal{A}} \min\{\mathbb{M}_t(l|i), \mathbb{M}_t(l|j)\}\\
		& \le 1 - \sum_{l \in \mathcal{S}\times \mathcal{A}} \tau = 1 - SA\tau\\ 
		& \le 1 - SA \epsilon \iota,
    \end{align*}
    where the first equality derives from Property~\ref{property:def} in Appendix~\ref{appendix:dobrushin}, while the last two inequalities follow from the definition of \(\tau\).\\
    Since the Dobrushin coefficient is known to be an upper bound on the modulus of the second largest eigenvalue (Property~\ref{property:eigenUpper} in Appendix~\ref{appendix:dobrushin}), \(\forall t\) we have that:
    \begin{align*}
	|\lambda(\mathbb{M}_t)|\le \rho(\mathbb{M}_t) \le 1 - SA\epsilon \iota.
	\end{align*}
	From now on, we will thus use:
	\begin{align}\label{elp:lambdaMaxDef}
		\lambda_{\max}:=1 - SA\epsilon \iota
    \end{align}
	to denote the maximum magnitude \(\lambda(\mathbb{M}_t)\) can have for any \(t\).\\

	We can now tackle the problem of the process not starting from its invariant distribution. The additional factor influencing the estimation error in this case is how fast the Markov chain transitions from the initial distribution to the stationary one. For this purpose, we will use Lemma~\ref{lemma:distributionBoundNonStat} appearing in Appendix~\ref{appendix:estimatErrBounds}, which bounds the estimation error of a categorical distribution from samples coming from a Markov chain not starting from the stationarity. If we denote with \(\bm{\nu} \in \Delta(\mathcal{A} \times \mathcal{S})\) the starting action-state distribution and with \(\bm{d}_{AS} \in \Delta(\mathcal{A} \times \mathcal{S})\) the stationary distribution induced by the used policy,  Lemma~\ref{lemma:distributionBoundNonStat} allows us to write that, with probability at least \(1 - \delta\), it holds:
	\begin{align}
		\|\bm{d}_{A^2O^2} - \widehat{\bm{d}}_{A^2O^2}\|_2 \le \sqrt{\Big( \frac{1+\lambda_{\max}}{1-\lambda_{\max}}\Big)\frac{2C(\bm{\nu},\bm{d}_S)+ 5\log(C(\bm{\nu},\bm{d}_S)/\delta)}{2n}}\label{elp:estimationBoundFromProp},
	\end{align}
 where \(C(\bm{\nu},\bm{d}_{AS})\) is defined as \(C(\bm{\nu},\bm{d}_{AS})=\left\| \frac{\bm{\nu}}{\bm{d}_{AS}}\right\|_{\infty}\) and \(\frac{\bm{\nu}}{\bm{d}_{AS}}\) is a vector of size \(AS\) containing the element-wise ratio of the elements of the two vectors. We can easily see that when \(\bm{\nu} = \bm{d}_{AS}\), we have \(C(\bm{\nu},\bm{d}_{AS})=1\) thus resorting to the result appearing in Lemma~\ref{lemma:distributionBoundStat}.
	
	Summarizing the obtained result and including those from Lemma~\ref{lemma:estimationErrorBound}, we get with probability at least \(1 - \delta\):
	\begin{align}
		\|\mathbb{T} - \widehat{\mathbb{T}}\|_{F} & \le \frac{2}{\alpha^2 d_{\min} \iota } \sqrt{\Big( \frac{1+\lambda_{\max}}{1-\lambda_{\max}}\Big)\frac{SA(2C(\bm{\nu},\bm{d}_{AS})+ 5\log(C(\bm{\nu},\bm{d}_{AS})/\delta))}{n}} \notag \\
        & \le \frac{2}{\alpha^2 \epsilon \iota } \sqrt{\Big( \frac{1+\lambda_{\max}}{1-\lambda_{\max}}\Big)\frac{SA(2C(\bm{\nu},\bm{d}_{AS})+ 5\log(C(\bm{\nu},\bm{d}_{AS})/\delta))}{n}} \label{lp:dminEps}\\
		& \le \frac{2}{\alpha^2 (\epsilon \iota)^{3/2} } \sqrt{\frac{2C(\bm{\nu},\bm{d}_{AS})+ 5\log(C(\bm{\nu},\bm{d}_{AS})/\delta))}{n}}
	\end{align}
 where in line~\eqref{lp:dminEps} we simply used that \(\epsilon \le d_{\min}\), while in the last step we used \(1 + \lambda_{\max} \le 2\) at the numerator and we substituted the definition of \(\lambda_{\max}\) appearing in~\eqref{elp:lambdaMaxDef} at the denominator. 
This leads to removing the explicit dependence on \(S\) and \(A\) from the bound at the cost of having a further \(\sqrt{\epsilon\iota}\) term in the final result. This last step concludes the proof.
\end{proof}
\vspace{0.2cm}

\subsection{Proof of Theorem~\ref{theorem:algorithmBound}}\label{appendix:theoremProof}
In this subsection, we will present the proof related to the main result presented in Theorem~\ref{theorem:algorithmBound}. The proof shares some steps of the proof presented in the work of~\citet{zhou2021regime}, which presents a similar setting.\\
We report the statement of the theorem for clarity.

\algorithmBound*
where \(C_1\) and \(C_2\) are constants used to report the bound in a more compact form, while \(D\) represents the diameter of the obtained belief MDP.
\begin{proof}
	We recall here the definition of regret as reported in Equation~\eqref{def:regret} appearing in the main paper:
	\begin{align}\label{eq:regret01}
		\mathcal{R}_T := T\rho^* - \sum_{t=1}^{T} r_t^\pi(o_t) = \sum_{t=1}^{T}(\rho^* - \mathbb{E}^{\pi}[r(O_t)|\mathcal{F}_{t-1}]) + \sum_{t=1}^{T}(\mathbb{E}^{\pi}[r(O_t)|\mathcal{F}_{t-1}] - r(o_t)),
	\end{align}
	where the expectation \(\mathbb{E}^\pi\) is taken w.r.t. the true transition matrices \(\mathbb{T}\), the true observation matrices \(\mathbb{O}\) and the stochasticity of the policy \(\pi\). We can observe that the second term in the summation is a martingale. Indeed, by defining a stochastic process as:
	\begin{align*}
		X_0=0, X_t = \sum_{l=1}^t (\mathbb{E}^{\pi}[r(O_l)|\mathcal{F}_{l-1}] - r(o_l)),
	\end{align*}
	we can easily see that \(X_t\) is a martingale. This allows us to apply the Azuma-Hoeffding inequality~\cite{azuma1967weighted} and obtain that with probability at least \(1 - \delta/4\) we have:
	\begin{align}\label{tp:martingale00}
		\sum_{t=1}^T (\mathbb{E}^{\pi}[r(O_t)|\mathcal{F}_{t-1}] - r(o_t)) \le \sqrt{2T\log(4/\delta)}.
	\end{align}
	We will now define some quantities that will be useful for the following analysis. We recall here the definition reported in Equation~\eqref{def:stateActionRew} of the expected reward of an action \(a_t\) assuming to be in state \(s_t\) as:
	\begin{align*}
		\mu(s_t, a_t) = \sum_{o \in \mathcal{O}}r(o) \mathbb{O}_{a_t}(o|s_t) = \bm{r}^\top \mathbb{O}_{a_t}(\cdot|s_t).
	\end{align*} 
    Therefore, we can define the expected reward given a belief state \(b_t\) at time \(t\) when taking action \(a_t\) as:
	\begin{align}
		g(b_t, a_t) = \sum_{s \in \mathcal{S}}\mu(s, a_t) b_t(s) = \bm{\mu}(a_t)b_t = \bm{r}^\top \mathbb{O}_{a_t}b_t,\label{tp:gDefinition}
	\end{align}
	where the last equalities define the expression in matrix notation, with \(\bm{\mu}(a_t)\) being a vector of dimension \(S\) containing the quantity \(\mu(s,a_t)\:\: \forall s \in \mathcal{S}\).
	Since action \(A_t\) is adapted to the filtration \(\mathcal{F}_{t-1}\), we have:
	\begin{align*}
		\mathbb{E}^\pi[\mu(S_t, A_t)|\mathcal{F}_{t-1}] = g(b_t,A_t),
	\end{align*}
	where the belief \(b_t\) is computed using the true transition and observation matrices and actions are taken according to policy \(\pi\). We will instead denote with:
	\begin{align*}
		\mathbb{E}^\pi_k[\mu(S_t, A_t)|\mathcal{F}_{t-1}] = g(b_t^k,A_t)
	\end{align*}
	the expected instantaneous reward assuming to have computed the belief using the estimate of the transition probability employed during the \(k\)-th episode. We refer to this estimate as \(\mathbb{T}_k\).\\
	Given the defined quantities, we can rewrite the first term of~\eqref{eq:regret01} as:
	\begin{align}
		\sum_{t=1}^{T}(\rho^* - \mathbb{E}^{\pi}[r(O_t)|\mathcal{F}_{t-1}]) = \sum_{t=1}^{T}(\rho^* - \mathbb{E}^{\pi}[\mu(S_t, A_t)|\mathcal{F}_{t-1}]) = \sum_{t=1}^{T}(\rho^* - g(b_t,A_t)).
	\end{align}
	Given that the algorithm proceeds in episodes of increasing length, we can rewrite the previous expression by denoting with \([N_k]\) the steps belonging to the \(k\)-th phase:
	\begin{align}
		\sum_{k=1}^K \sum_{t \in [N_k]} (\rho^* - g(b_t,A_t)),
	\end{align}
	where from this summation, we exclude the episode \(k=0\) that is used to initialize the POMDP parameters and for which the suffered regret corresponds to the length of episode \(N_0=T_0\). At the beginning of each episode, the proposed algorithm defines a confidence \(\mathcal{C}_k(\delta_k)\) such that the real POMDP instance \(\mathcal{Q}\) is contained in high probability in the region, namely \(P(\mathcal{Q} \in \mathcal{C}_k(\delta_k)) \ge 1 - \delta_k\), with \(\delta_k \eqqcolon \delta/k^3\). We assume that the provided oracle is able to return the value function \(v_k\), the average reward \(\rho^k\), and the transition matrix \(\mathbb{T}_k\) corresponding to the most optimistic POMDP instance \(\mathcal{Q}^{(k)}\) belonging to the confidence region \(\mathcal{C}_k(\delta_k)\).\\
    We take into account now two possible events: the \emph{good event} considers the case where for all episodes \(k\), the true POMDP is contained in the confidence sets \(\{\mathcal{C}_k(\delta_k)\}_{k=1}^K\); the \emph{failure event} denotes the complementary event.\\
    We can thus derive:
	\begin{align*}
		P(\mathcal{Q} \notin \mathcal{C}_k(\delta_k), \text{for some k}) \le \sum_{k=1}^K \delta_k = \sum_{k=1}^K \frac{\delta}{k^3} \le \frac{3}{2}\delta.
	\end{align*}
	From this relation, with probability at least \(1 - \frac{3}{2}\delta\), the \emph{good event} holds. When the \emph{good event} holds, we have that the \(\rho^* \le \rho^k\) for any \(k\) since the optimal average reward is taken from the optimistic POMDP \(\mathcal{Q}^{(k)}\).\\
	We can now bound the regret under the \emph{good event} during the \(k\) phases.
	\begin{align}
		\sum_{k=1}^{K} \sum_{t \in [N_k]} (\rho^* - g(b_t,A_t)) & \le \sum_{k=1}^{K} \sum_{t \in [N_k]} (\rho^k - g(b_t,A_t)) \notag \\
		& = \sum_{k=1}^{K} \sum_{t \in [N_k]} (\rho^k - g(b_t^k,A_t)) + (g(b_t^k,A_t) - g(b_t,A_t))\label{tp:001}.
	\end{align}
	We will proceed by initially bounding the first term of~\eqref{tp:001} and then the second one. The first term can be bounded by using the Bellman expectation equation (see Equation~\eqref{appendix:BellExpectation}) for the optimistic belief MDP \(\mathcal{Q}^{(k)}\). Before going on, we introduce the vector \(U_k(\cdot|b_t^k, A_t) = P_{\mathbb{T}_k}(b_{t+1}\in \cdot|b_t^k, A_t)\) which defines the probability of each possible future belief \(b_{t+1}\) conditioned on the actual belief \(b_t^k\), on the action taken \(A_t\) and with transition kernel given by the estimated \(\mathbb{T}_k\). We rewrite the Bellman expectation equation under policy \(\pi^{(k)}\) as follows:
	\begin{align}
		\rho^k + v_k(b_t^k) & = \underset{a \sim \pi^{(k)}(\cdot|b_t^k)}{\mathbb{E}} \left[g(b_t^k,a) + \int_{b_{t+1} \in \mathcal{B}} v_k(b_{t+1})U_k(\,db_{t+1}|b_t^k,a)\right]\notag \\
		& = \underset{a \sim \pi^{(k)}(\cdot|b_t^k)}{\mathbb{E}} \left[g(b_t^k,a) + \langle U_k(\cdot|b_t^k, a), v_k(\cdot) \rangle\right]\label{eq:bellmanDefExp}
	\end{align}
	Given that the value function \(v_k\) satisfies the Bellman expectation Equation, the same would hold for a shifted version \(v_k + c\bm{1}\). From this consideration, we can assume that 
	\(\|v_k\|_\infty \le \text{span}(v_k)/2\). By using the result in Proposition~\ref{prop:uniformBoundBias} taken from~\citet{zhou2021regime}, we are able to bound the span of \(v_k\), where the span is defined  as \(span(v_k):=\max_{b \in \mathcal{B}}v_k(b) - \min_{b \in \mathcal{B}}v_k(b)\). The span is bounded by a finite quantity \(D\) representing the diameter of the POMDP. We can thus write:
	\begin{align}\label{tp:biasModuleBound}
		\|v_k\|_\infty \le \frac{\text{span}(v_k)}{2} \le \frac{D}{2}.
	\end{align}
	Taking into account the previous considerations, we can write for the first term of~\eqref{tp:001}:
	\begin{align}
		& \sum_{k=1}^{K} \sum_{t \in [N_k]} (\rho^k - g(b_t^k,A_t)) \notag \\
		= & \sum_{k=1}^{K} \sum_{t \in [N_k]} \left(-v_k(b_t^k) + \langle U_k(\cdot|b_t^k, A_t), v_k(\cdot) \rangle \right) + \left(\underset{a \sim \pi^{(k)}(\cdot|b_t^k)}{\mathbb{E}}\left[\langle U_k(\cdot|b_t^k, a), v_k(\cdot) \rangle)\right] - \langle U_k(\cdot|b_t^k, A_t), v_k(\cdot) \rangle \right) \notag \\
		& \qquad \qquad + \left(\underset{a \sim \pi^{(k)}(\cdot|b_t^k)}{\mathbb{E}}\left[g(b_t^k,a)\right] - g(b_t^k,A_t)\right) \notag \\
		= & \sum_{k=1}^{K} \sum_{t \in [N_k]} \left(-v_k(b_t^k) + \langle U_k(\cdot|b_t^k, A_t), v_k(\cdot) \rangle \right) + \sum_{k=1}^{K} \sum_{t \in [N_k]} \left(\underset{a \sim \pi^{(k)}(\cdot|b_t^k)}{\mathbb{E}}\left[\langle U_k(\cdot|b_t^k, a), v_k(\cdot) \rangle)\right] - \langle U_k(\cdot|b_t^k, A_t), v_k(\cdot) \rangle \right) \notag \\
		& \qquad \qquad + \sum_{k=1}^{K} \sum_{t \in [N_k]} \left(\underset{a \sim \pi^{(k)}(\cdot|b_t^k)}{\mathbb{E}}\left[g(b_t^k,a)\right] - g(b_t^k,A_t)\right),\label{tp:threesums}
	\end{align}
	where the first equality is obtained by substituting \(\rho^k\) as appears in the Bellman Equation in line~\eqref{eq:bellmanDefExp} and by adding and subtracting the quantity \(\langle U_k(\cdot|b_t^k, A_t), v_k(\cdot) \rangle\) for all the time steps. From the last equality, we observe three different blocks of summations. We focus on the last two: it can be shown that they are both martingale sequences (see Proposition 4 in~\citet{zhou2021regime} for the second term), that can be bounded as follows.\\ 
	For the second term in~\eqref{tp:threesums}, we have with probability at least \(1 - \delta/4\):
	\begin{align}
		\sum_{k=1}^{K} \sum_{t \in [N_k]} \left(\underset{a \sim \pi^{(k)}(\cdot|b_t^k)}{\mathbb{E}}\left[\langle U_k(\cdot|b_t^k, a), v_k(\cdot) \rangle)\right] - \langle U_k(\cdot|b_t^k, A_t), v_k(\cdot) \rangle \right) \le D\sqrt{2T \log \left(\frac{4}{\delta}\right)}.\label{tp:martingale01}
	\end{align}
	Analogously, for the third term:
	\begin{align}
		\sum_{k=1}^{K} \sum_{t \in [N_k]} \left(\underset{a \sim \pi^{(k)}(\cdot|b_t^k)}{\mathbb{E}}\left[g(b_t^k,a)\right] - g(b_t^k,A_t)\right) \le \sqrt{2T \log \left(\frac{4}{\delta}\right)}.\label{tp:martingale02}
	\end{align}
	Let us now focus on the first term appearing in~\eqref{tp:threesums}. We can rewrite it as:
	\begin{align}
		& \sum_{k=1}^{K} \sum_{t \in [N_k]} \left(-v_k(b_t^k) + \langle U_k(\cdot|b_t^k, A_t), v_k(\cdot) \rangle \right) \notag\\
		= & \sum_{k=1}^{K} \sum_{t \in [N_k]} \Big(-v_k(b_t^k) + \langle U(\cdot|b_t^k, A_t), v_k(\cdot) \rangle \Big) + \Big(\langle U_k(\cdot|b_t^k, A_t) - U(\cdot|b_t^k, A_t), v_k(\cdot) \rangle\Big)\label{tp:boundVQ},
	\end{align}
	where we recall that \(U(\cdot|b_t^k, A_t)\) represents the probability distribution over the belief at the next step \(t+1\) under the true transition matrix \(\mathbb{T}\), while \(U_k(\cdot|b_t^k, A_t)\) represents the same probability distribution under the estimated transition matrix \(\mathbb{T}_k\).\\
	For the first term of~\eqref{tp:boundVQ}, we can write:
	\begin{align}
		\sum_{k=1}^{K} \sum_{t \in [N_k]} \left(-v_k(b_t^k) + \langle U(\cdot|b_t^k, A_t), v_k(\cdot) \rangle \right)
		= \; & \sum_{k=1}^{K} \sum_{t \in [N_k]} \left(-v_k(b_t^k) + v_k(b_{t+1}^k)\right) + \left(-v_k(b_{t+1}^k) + \langle U(\cdot|b_t^k, A_t), v_k(\cdot) \rangle \right)\\
		= \; & \sum_{k=1}^{K} -v_k(b_{s_k}^k) + v_k(b_{e_k+1}^k) + \sum_{k=1}^{K} \sum_{t \in [N_k]} \mathbb{E}^\pi [v_k(b_{t+1}^k|\mathcal{F}_t)] - v_k(b_{t+1}^k),
	\end{align}
	where the first term in the last equality is a telescopic summation, thus resulting in the first value \(v_k(b_{s_k}^k)\) and the last value \(v_k(b_{e_k+1}^k)\) of the bias function \(v_k\) in the exploitation phase \(k\). The second term in the last equality is instead obtained by showing that:
	\begin{align*}
		\langle U(\cdot|b_t^k, A_t), v_k(\cdot) \rangle = \int_{b_{t+1} \in \mathcal{B}} v_k(b_{t+1})U(\,db_{t+1}|b_t^k,A_t) = \mathbb{E}^\pi[v_k(b_{t+1}^k|b_t^k)] =  \mathbb{E}^\pi [v_k(b_{t+1}^k|\mathcal{F}_t)].
	\end{align*}
	By recalling Proposition~\ref{prop:uniformBoundBias} to bound the span of the bias function, we get:
	\begin{align}
		\sum_{k=1}^K -v_k(b_{s_k}^k) + v_k(b_{e_k+1}^k) \le \sum_{k=1}^K D = KD.\label{tp:boundKD}
	\end{align}
	By applying instead analogous considerations as those used for bounding~\eqref{tp:martingale01}, we can state that with probability at least \(1 - \delta/4\):
	\begin{align}
		\sum_{k=1}^{K} \sum_{t \in [N_k]} \mathbb{E}^\pi [v_k(b_{t+1}^k|\mathcal{F}_t)] - v_k(b_{t+1}^k) \le D \sqrt{2T \log \left(\frac{4}{\delta}\right)}.\label{tp:martingale03}
	\end{align}
	By putting together the previous considerations, we can bound the first term appearing in~\eqref{tp:boundVQ} as:
	\begin{align}
		\sum_{k=1}^{K} \sum_{t \in [N_k]} \left(-v_k(b_t^k) + \langle U(\cdot|b_t^k, A_t), v_k(\cdot) \rangle \right) \le KD + D \sqrt{2T \log \left(\frac{4}{\delta}\right)}.
	\end{align}
	We can proceed now in bounding the second term of~\eqref{tp:boundVQ}. Before that, we introduce functions \(H(b_t, a_t, o_t)\) and \(H_k(b_t, a_t, o_t)\) which deterministically return the belief at the next time step \(b_{t+1}\) given the current belief \(b_t\), the action taken \(a_t\) and the received observation \(o_t\) using the real \(\mathbb{T}_{a_t}\) and the estimated transition matrix \(\mathbb{T}_{a_tk}\), respectively.\\
	Let us analyze each quantity appearing in the summation of the second term of~\eqref{tp:boundVQ}:
	\begin{align}
		\langle U_k(\cdot|b_t^k, A_t) - U(\cdot|b_t^k, A_t), v_k(\cdot) \rangle
		\le \; & \Bigg\lvert \int_{\mathcal{B}} v_k(b')U_k(\,db'|b_t^k,A_t) - \int_{\mathcal{B}} v_k(b')U(\,db'|b_t^k,A_t)\Bigg\rvert \notag \\
		= \; & \Bigg\lvert \sum_{o_t \in \mathcal{O}} v_k(H_k(b_t^k,A_t,o_t)) P(o_t|b_t^k, A_t) - \sum_{o_t \in \mathcal{O}} v_k(H(b_t^k,A_t,o_t)) P(o_t|b_t^k, A_t) \Bigg\rvert \label{tp:Qdiff01}\\
		= \; & \Bigg\lvert \sum_{o_t \in \mathcal{O}} \left[ v_k(H_k(b_t^k,A_t,o_t)) - v_k(H(b_t^k,A_t,o_t)) \right] P(o_t|b_t^k, A_t)\Bigg\rvert \notag\\
		\le \; & \sum_{o_t \in \mathcal{O}} \Big\lvert v_k(H_k(b_t^k,A_t,o_t)) - v_k(H(b_t^k,A_t,o_t)) \Big\rvert P(o_t|b_t^k, A_t) \notag\\
		\le \; & \sum_{o_t \in \mathcal{O}} \Big\lvert \frac{D}{2}\big(H_k(b_t^k,A_t,o_t) -H(b_t^k,A_t,o_t)\big) \Big\rvert P(o_t|b_t^k, A_t)\label{tp:Qdiff02}\\
		\le \; & \sum_{o_t \in \mathcal{O}} \frac{D}{2} \left(L \|\mathbb{T}_{A_t} - \mathbb{T}_{A_tk}\|_F \right) P(o_t|b_t^k, A_t)\label{tp:Qdiff03}\\
		= \; & \frac{D L}{2} \|\mathbb{T}_{A_t} - \mathbb{T}_{A_tk}\|_F,\label{tp:Qdiff04}
	\end{align}
	where in line~\eqref{tp:Qdiff01} we have explicitly decoupled the stochasticity induced by the observation from the deterministic update of the belief \(b'\) at the next step through the \(H\) and \(H_k\) functions. Line~\eqref{tp:Qdiff02} is obtained from the bound on the bias function as defined in~\eqref{tp:biasModuleBound} and Holder's inequality, while line~\eqref{tp:Qdiff03} is obtained from Proposition~\ref{prop:beliefErrorPOMDP}, which has been derived from the original result of~\citet{deCastroConsistent2017}.
	
	Given the results obtained so far, we are able to bound the first term of~\eqref{tp:001}. 
	By combining the bounds on the martingale sequences~\eqref{tp:martingale01},~\eqref{tp:martingale02},~\eqref{tp:martingale03} with~\eqref{tp:boundKD} and~\eqref{tp:Qdiff04}, we get: 
	\begin{align}
		& \sum_{k=1}^{K} \sum_{t \in [N_k]} (\rho^k - g(b_t^k,A_t)) \le KD + 2 D \sqrt{2T \log \left(\frac{4}{\delta}\right)} + \sqrt{2T \log \left(\frac{4}{\delta}\right)} + \sum_{k=1}^{K} \sum_{t \in [N_k]} \frac{D L}{2} \|\mathbb{T}_{A_t} - \mathbb{T}_{A_tk}\|_F
	\end{align}
	We can now proceed in analyzing the second term of~\eqref{tp:001}:
	\begin{align*}
		\sum_{k=1}^{K} \sum_{t \in [N_k]} (g(b_t^k,A_t) -  g(b_t,A_t)) = \; & \sum_{k=1}^{K} \sum_{t \in [N_k]} (\bm{r}^\top \mathbb{O}_{A_t}b_t^k -  \bm{r}^\top \mathbb{O}_{A_t}b_t)\\
		= \; & \sum_{k=1}^{K} \sum_{t \in [N_k]} \bm{r}^\top \mathbb{O}_{A_t} (b_t^k - b_t)\\
        \le \; & \sum_{k=1}^{K} \sum_{t \in [N_k]} \|\bm{r}^\top \mathbb{O}_{A_t}\|_\infty \|b_t^k - b_t\|_1\\
		\le \; & \sum_{k=1}^{K} \sum_{t \in [N_k]} \|b_t^k - b_t\|_1 \\
            \le \; & \sum_{k=1}^{K} \sum_{t \in [N_k]} L \|\mathbb{T} - \mathbb{T}_{k}\|_F
	\end{align*}
	where in the first equality, we use the definition of the function \(g\) in~\eqref{tp:gDefinition}. In the first inequality we use Holder's inequality, in the second inequality we consider that \(\|\bm{r}^\top \mathbb{O}_{A_t}\|_\infty \le 1\). The final inequality derives from the bound on the belief of Proposition~\ref{prop:beliefErrorPOMDP}, where we used that \(\|\mathbb{T}_a - \mathbb{T}_{ak}\|_F \le \|\mathbb{T} - \mathbb{T}_{k}\|_F\) holds for all \(a \in \mathcal{A}\).\\
	We can finally put all the results together, thus having:
	\begin{align}
		\sum_{k=1}^{K} \sum_{t \in [N_k]} (\rho^* - g(b_t,A_t)
		\le \; & KD + 2 D \sqrt{2T \log \left(\frac{4}{\delta}\right)} + \sqrt{2T \log \left(\frac{4}{\delta}\right)} + \sum_{k=1}^{K} \sum_{t \in [N_k]} \frac{D L}{2} \|\mathbb{T}_{A_t} - \mathbb{T}_{A_tk}\|_F + L \|\mathbb{T} - \mathbb{T}_{k}\|_F \notag\\
		\le \; & KD + 2 D \sqrt{2T \log \left(\frac{4}{\delta}\right)} + \sqrt{2T \log \left(\frac{4}{\delta}\right)} + \sum_{k=1}^{K} \sum_{t \in [N_k]} \left(1 + \frac{D}{2}\right) L \|\mathbb{T} - \mathbb{T}_{k}\|_F,\label{tp:boundthreesums}
	\end{align}
	where in the last inequality we used again that \(\|\mathbb{T}_{A_t} - \mathbb{T}_{A_tk}\|_F \le \|\mathbb{T} - \mathbb{T}_k\|_F\).\\
	
	Let us now consider the last term appearing in line~\eqref{tp:boundthreesums}. The estimation error of the transition model could be bounded by using Lemma~\ref{lemma:extendedEstimationErrorBound} that uses Assumption~\ref{ass:minElem} and also does not assume that the process starts from its invariant distribution.\\
	We can thus rewrite the last term in~\eqref{tp:boundthreesums} as follows, holding with probability at least \(1 - \frac{3}{2}\delta\) (since it is the probability of the \emph{good event}):
    \begin{align}\label{tp:transMatBound}
        \sum_{k=1}^{K} \sum_{t \in [N_k]} \left(1 + \frac{D}{2}\right) L \|\mathbb{T} - \mathbb{T}_{k}\|_F \le \sum_{k=1}^{K} \sum_{t \in [N_k]}\frac{L \left(2 + D\right)}{\alpha^2 (\epsilon \iota)^{3/2} } \sqrt{\frac{2C_k+ 5\log(C_k/\delta_k)}{N_{k-1}}},
    \end{align}
    where \(N_{k-1}\) represents the number of samples collected during the \((k-1)\)-th episode, while \(C_k=C(\bm{\nu}_k, \bm{d}_{AS,k})\) where \(\bm{\nu}_k \in \Delta(\mathcal{A} \times \mathcal{S})\) and \(\bm{d}_{AS,k} \in \Delta(\mathcal{A} \times \mathcal{S})\) are respectively the action-state distribution at the beginning of the \(k\)-th episode and the action-state stationary distribution induced by policy \(\pi^{(k)}\).\\
	
    Let us recall now how the OAS-UCRL Algorithm handles the switch from one episode to the next one: it happens when the number of samples collected during the \(k\)-th episode reaches \(N_k = T_0 2^{k}\). By considering a total of \(K+1\) episodes (including \(k=0\)), they would last for a total number of steps given by:
	\begin{align}\label{tp:numberOfIter}
		\sum_{i=0}^K N_i = \sum_{i=0}^{K}T_0 2^i = T_0 (2^{K+1}-1).
	\end{align}
    By considering a total horizon of length \(T\), we have that:
	\begin{align}\label{tp:lmaxDef}
		K +1 \le \left\lceil \log_2\left(\frac{T}{T_0} + 1\right)\right\rceil.
	\end{align}
	Let us consider now the formulation appearing in~\eqref{tp:transMatBound} in terms of the length of each episode:
 	\begin{align}
		\sum_{k=1}^{K} \sum_{t \in [N_k]}\frac{L \left(2 + D\right)}{\alpha^2 (\epsilon \iota)^{3/2} } \sqrt{\frac{2C_k+ 5\log(C_k k^3/\delta)}{N_{k-1}}} 
		\le \; \frac{L \left(2 + D\right)}{\alpha^2 (\epsilon \iota)^{3/2} } \sqrt{2C_{\max}+ 5 \log(C_{\max} K^3/\delta)} \sum_{k=1}^{K} \sum_{t \in [N_k]} \sqrt{\frac{1}{N_{k-1}}}\label{tp:summationOnNk}
	\end{align}
	where we defined the quantity \(C_{\max}:=\max_{k \in [K]}C_k\) and reorganized the expression to isolate the constant terms.
	Let us focus now on the summation appearing in~\eqref{tp:summationOnNk}. We have:
\begin{align}
    \sum_{k=1}^{K} \sum_{t \in [N_k]} \sqrt{\frac{1}{N_{k-1}}} & \le \sum_{k=1}^{K} T_02^k \sqrt{\frac{1}{T_02^{k-1}}} = \sqrt{2 T_0} \sum_{k=1}^{K} \sqrt{2}^{k} \le \frac{\sqrt{2T_0}}{\sqrt{2}-1}\sqrt{2}^{K+1} \notag\\
    & \le \sqrt{2T_0}(1 +\sqrt{2})\sqrt{\frac{T}{T_0}+1}\label{tp:klog}\\
    & \le \sqrt{2T_0}(1 +\sqrt{2})\sqrt{\frac{2T}{T_0}}\notag \\
    & \le 6 \sqrt{T},\label{tp:sqrtTBound}
\end{align}
where the relations easily follow from simple algebraic manipulations while line~\eqref{tp:klog} uses the expression in~\eqref{tp:lmaxDef}.\\
Wrapping up the previous results and combining them with~\eqref{tp:martingale00}, we are finally able to write the formulation of the regret, holding with probability at least \(1 - \frac{5}{2}\delta\):
	\begin{align}
		\mathcal{R}_T & \le KD + 2 D \sqrt{2T \log \left(\frac{4}{\delta}\right)} + 2 \sqrt{2T \log \left(\frac{4}{\delta}\right)} + T_0 + \frac{6 L \left(2 + D\right)}{\alpha^2 (\epsilon \iota)^{3/2} } \sqrt{\big(2C_{\max}+ 5 \log(C_{\max} K^3/\delta)\big)T} \notag\\
		& \le \log_2(T)D + 4 D \sqrt{2T \log \left(\frac{4}{\delta}\right)} + T_0 + \frac{6L \left(2 + D\right)}{\alpha^2 (\epsilon \iota)^{3/2} } \sqrt{\big(2C_{\max}+ 5 \log(C_{\max}/\delta) + 15\log(\log T)\big)T}\label{tp:finalExtendedBound}
	\end{align}
	where we introduced the quantity \(T_0\) related to the regret suffered in the first episode (\(k=0\)) and we used that \(K \le \log(T)\) which can be derived from~\eqref{tp:lmaxDef}.
	By rearranging the bound and expressing it in a more compact form highlighting the most relevant terms, we have:
	\begin{align}
		\mathcal{R}_T & \le C_1\frac{S \left(2 + D\right)}{\alpha^2 (\epsilon \iota)^{3/2}} \sqrt{T \log(T/\delta)} + C_2,
	\end{align}
	where we make explicit the linear dependence on the number of states \(S\) derived from the definition of \(L\) and define constants \(C_1\) and \(C_2\), which are obtained by simplifying the expression in~\eqref{tp:finalExtendedBound}.\\
	This last expression completes the proof.
\end{proof}
\vspace{0.2cm}

\section{Geometric ergodicity and Dobrushin coefficient for Markov Chains}\label{appendix:dobrushin}
The objective of this section is to introduce the concept of geometric ergodicity for Markov chains and characterize the speed of convergence towards the stationary distribution.\\
We will focus here on Markov chains with finite state space \(|X| = \mathcal{X}\) and such that their corresponding transition matrix \(\mathbb{T}\) satisfies Assumption~\ref{ass:minElem}, thus having that each cell of the matrix is lower bounded by a quantity \(\epsilon\).\\
A key element that we need to introduce is the Dobrushin coefficient~\cite{krish2016partially}. Let us consider two distributions \(\nu, \bar{\nu} \in \Delta(\mathcal{X})\) and a transition matrix \(\mathbb{T}\) on \(\mathcal{X}\).
The Dobrushin coefficient \(\rho(\cdot)\) gives the following bound:
\begin{align}\label{eq:dobrushinContraction}
	\| \mathbb{T}^\top\nu - \mathbb{T}^\top\bar{\nu}\|_1 \le \rho(\mathbb{T})\|\nu - \bar{\nu}\|_1,
\end{align}
where \(\mathbb{T}^\top\) represents the transpose of matrix \(\mathbb{T}\). The inequality above says that the one-step evolution of two probability vectors induced by the same kernel \(\mathbb{T}\) can be bounded by the quantity on the right where scalar \(\rho(\mathbb{T}) \in [0,1]\) represents the Dobrushin coefficient. If \(\rho(\mathbb{T})\) is strictly smaller than 1, by iteratively applying the inequality, it is possible to ensure geometric convergence of the initial distance.\\
Among the relevant properties associated with this quantity (we refer to~\citet{krish2016partially} for more details), we report the ones that will be relevant for our purposes:
\begin{enumerate}
	\item \(\rho(\mathbb{T}) = 1 - \min_{i,j \in \mathcal{X}} \sum_{l \in \mathcal{X}} \min\{\mathbb{T}(l|i), \mathbb{T}(l|j)\}\)\label{property:def}.
	\item \(|\lambda(\mathbb{T})| \le \rho(\mathbb{T})\). That is, the Dobrushin coefficient upper bounds the modulus of the second largest eigenvalue of matrix \(\mathbb{T}\)\label{property:eigenUpper}.
\end{enumerate}
Under Assumption~\ref{ass:minElem}, we can apply Property~\ref{property:def} to set an upper bound on the Dobrushin coefficient \(\rho(\mathbb{T})\):
\begin{align*}
	\rho(\mathbb{T}) & = 1 - \min_{i,j \in \mathcal{X}} \sum_{l \in \mathcal{X}} \min\{\mathbb{T}(l|i), \mathbb{T}(l|j)\}\\
	& \le 1 - \sum_{l \in \mathcal{X}} \epsilon = 1 - X \epsilon
\end{align*}
where the inequality derives from the assumption.\\
Since we have set an upper bound on \(\rho(\mathbb{T})\), from Property~\ref{property:eigenUpper}, it also follows that this is an upper bound also for the second largest eigenvalue:
\begin{align*}
	|\lambda(\mathbb{T})| \le \rho(\mathbb{T}) \le 1 - X\epsilon
\end{align*}
The considerations above show that the Dobrushin coefficient is useful when it is strictly less than 1. On the contrary, for sparse matrices, \(\rho(\mathbb{T})\) is typically equal to 1, thus only providing trivial upper bounds for \(|\lambda(\mathbb{T})|\).
\vspace{0.2cm}

\section{OAS Estimation Procedure under Continuous Observations}\label{appendix:continuousObsSpace}
The objective of this section is to present how the OAS estimation procedure described in the main text should be adapted when the observation space is no longer finite.\\
When the observation space is continuous and known, we can as well apply the OAS estimation procedure but we first need to discretize the observation space. Analogously to what is suggested in~\citet{anonymous}, continuous distributions can be converted into multinomial ones. The discretization process goes on by defining a number \(U\) of distinct consecutive intervals and the split points characterizing the partition of the observation space. The split points should be the same for any continuous distribution \(P(\cdot|a,s)\) in the observation space, for each \(a \in \mathcal{A}, s \in \mathcal{S}\). After the discretization step, each observation matrix \(\mathbb{O}_a\) has dimension \(U \times S\) and the value contained in a cell \(\mathbb{O}_a(u|s)\) will be:
\begin{align*}
	\mathbb{O}_a(u|s) = P(u \in \mathcal{I}_{ab}|s,a) = \int_{u_a}^{u_b}P(do|s,a)do,
\end{align*}
where \(u_a < u_b\) are the extremes of interval \(\mathcal{I}_{ab}\).
If the integrals can be correctly computed, no discretization errors will be introduced with this procedure. The newly generated observation matrices will need to satisfy Assumption~\ref{ass:weaklyRev}.\\
It is not an easy problem to find the splitting points in the observation space that allow defining observation matrices with a high minimum singular value \(\sigma_{\min}\); however, it is possible to devise an optimization procedure that proceeds over multiple runs.\\ For each run, a set of splitting points is randomly chosen, and the new observation matrices are computed together with their minimum singular value. At the end of the trials, the splitting points of the run showing the highest \(\sigma_{\min}\) will be chosen.\\
Finally, during the estimation procedure in Algorithm~\ref{alg:estimationAlgorithm}, the vector of counts \(\bm{N}\) has dimension \(U^2A^2\) and, for each tuple \((a,a',u,u')\), it will increment by 1 the right cell, based on the intervals to which \(u\) and \(u'\) belong.
\vspace{0.2cm}

\section{Estimation Error Bounds of Stationary Distributions induced by Markov Chains}\label{appendix:estimatErrBounds}

In this section, we are going to show the main concentration results that are used in the proof of Lemma~\ref{lemma:estimationErrorBound}.\\
Lemma~\ref{lemma:distributionBoundStat} shows the convergence guarantees when the estimation comes from a stationary chain (the initial distribution coincides with the stationary one) while Lemma~\ref{lemma:distributionBoundNonStat} considers the case when the chain starts from an arbitrary distribution.

\begin{restatable}[\textbf{Stationary Distribution under Stationary Chain}]{lemma}{distributionBoundStat}\label{lemma:distributionBoundStat}
	Let us assume to have a POMDP instance \(\mathcal{Q}\) with a finite number of states \(\mathcal{S}\) and a finite number of actions \(\mathcal{A}\). Let \(\pi \in \mathcal{P}\) be the belief-based policy used on \(\mathcal{Q}\) and let us assume that the induced stationary distribution \(\bm{d}_{AS} \in \Delta(\mathcal{A} \times \mathcal{S})\) is ergodic, thus having that \(\bm{d}_{AS}(a,s) > 0 \; \forall a \in \mathcal{A}, s \in \mathcal{S}\). Let Assumption~\ref{ass:weaklyRev} hold. If the process starts from its stationary distribution \(\bm{d}_{AS}\), then by using the estimation procedure defined in Algorithm~\ref{alg:estimationAlgorithm}, with probability at least \(1 - \delta\), it holds that:
    \begin{align*}
        \|\bm{d}_{A^2O^2} - \widehat{\bm{d}}_{A^2O^2}\|_2 \le \sqrt{\Big( \frac{1+\lambda_{\max}}{1-\lambda_{\max}}\Big)\frac{(2+ 5\log(1/\delta))}{2 n}},
    \end{align*}
    where \(\lambda_{\max}\) denotes the maximum spectral gap associated with the inhomogeneous chain, while \(n\) denotes the number of tuples of the form \((a,a',o,o')\) collected during the interaction with the environment.
\end{restatable}
\begin{proof}
Since the chain is stationary, it also follows that the initial distribution \(\bm{d}_{A^2O^2}\) coincides with the stationary one.\\
To prove this lemma, we combine the results on the difference of probability vectors presented in~\citet{kontorovic2012uniform} (see their Section 2.4) with the new concentration results of~\citet{fanHoeffding2021} and obtain the following:
\begin{align}\label{lemmaproof:probNorm2Bound}
        \mathbb{P} \big( \|\bm{d}_{A^2O^2} - \widehat{\bm{d}}_{A^2O^2}\|_2 > \mathbb{E} \|\bm{d}_{A^2O^2} - \widehat{\bm{d}}_{A^2O^2}\|_2 + \epsilon \big) \le \exp \Big(-\frac{1 - \lambda_{\max}}{1 + \lambda_{\max}} 2 n \epsilon^2\Big).
\end{align}

To proceed with this part, we need to bound \(\mathbb{E} \|\bm{d}_{A^2O^2} - \widehat{\bm{d}}_{A^2O^2}\|_2\). This is done in Lemma~\ref{lemma:expectNorm2diffBound} that shows that the following relation holds: 
\begin{align*}
    \mathbb{E} \|\bm{d}_{A^2O^2} - \widehat{\bm{d}}_{A^2O^2}\|_2 \le \sqrt{\frac{1 + \lambda_{\max}}{n(1 - \lambda_{\max})}}.
\end{align*}
The presented result allows us to rewrite the relation in~\eqref{lemmaproof:probNorm2Bound} as:
\begin{align*}
    \mathbb{P} \Bigg( \|\bm{d}_{A^2O^2} - \widehat{\bm{d}}_{A^2O^2}\|_2 > \sqrt{\frac{1 + \lambda_{\max}}{n(1 - \lambda_{\max})}} + \epsilon \Bigg) \le \exp \Big(-\frac{1 - \lambda}{1 + \lambda} 2 n \epsilon^2\Big).
\end{align*}
By setting the right-hand side of the bound equal to \(\delta\), and rearranging the terms, we are able to get the final result of the Lemma. 
\end{proof}

We report here the claim and the proof of Lemma~\ref{lemma:expectNorm2diffBound} which is used to demonstrate the lemma above.

\begin{lemma}\label{lemma:expectNorm2diffBound}
Under the assumption of starting from a stationary distribution \(\bm{d}_{AS} \in \Delta(\mathcal{A} \times \mathcal{S})\), we have:
\begin{align*}
    \mathbb{E} \|\bm{d}_{A^2O^2} - \widehat{\bm{d}}_{A^2O^2}\|_2 \le \sqrt{\frac{1 + \lambda_{\max}}{n(1 - \lambda_{\max})}}
\end{align*}
\end{lemma}
\begin{proof}
The proof largely follows from~\citet{kontorovic2012uniform}. To reduce clutter, we will omit the pedices \(A^2O^2\) when present. By Jensen's inequality, we have:
\begin{align}
    (\mathbb{E} \|\bm{d} - \widehat{\bm{d}}\|_2)^2 & \le \mathbb{E}\big[ \|\bm{d} - \widehat{\bm{d}}\|_2^2\big ] \notag \\
    & = \mathbb{E}\Big[ \sum_{y \in [A^2O^2]} |\bm{d}_y - \widehat{\bm{d}}_y |^2\Big] \notag \\
    & = \sum_{y \in [A^2O^2]} \mathbb{E} \big[ (\bm{d}_y - \widehat{\bm{d}}_y)^2\big] \notag  \\
    & = \sum_{y \in [A^2O^2]} var \big[ \widehat{\bm{d}}_y \big], \label{lab:varDef}
\end{align}
where \(\bm{d}_y\) represents the \(y\)-th element of vector \(\bm{d}\) and it corresponds to the parameter of a Bernoulli random variable, 
\(\widehat{\bm{d}}_y\) denotes its estimator while \(var \big[ \widehat{\bm{d}}_y \big]\) represents the variance of the estimator. By denoting the quantity \(\sum_{i=1}^{n}\mathds{1}\{Y=y\}\), where \(\mathds{1}\) is the indicator function, from~\citet{fanHoeffding2021} we can state that this quantity is sub-Gaussian with variance proxy given by:
\begin{align*}
    \Big(\frac{1 + \lambda_{\max}}{1- \lambda_{\max}}\Big) n \bm{d}_y (1 - \bm{d}_y),
\end{align*}
with \(\lambda_{\max}\) being the modulus of the second largest eigenvalue of the inhomogeneous Markov chain.\\
From this consideration, it also follows that:
\begin{align*}
    var \big[ \widehat{\bm{d}}_y \big] \le \Big(\frac{1 + \lambda_{\max}}{1- \lambda_{\max}}\Big) \frac{\bm{d}_y}{n},
\end{align*}
that slightly improves over the result in~\citet{kontorovic2012uniform}. Reusing this relation in line~\eqref{lab:varDef} we get:
\begin{align*}
    (\mathbb{E} \|\bm{d} - \widehat{\bm{d}}\|_2)^2 & \le \sum_{y \in [A^2O^2]} var \big[ \widehat{\bm{d}}_y \big] \\
    & \le \sum_{y \in [A^2O^2]} \Big(\frac{1 + \lambda_{\max}}{1- \lambda_{\max}}\Big) \frac{\bm{d}_y}{n}\\
    & = \Big(\frac{1 + \lambda_{\max}}{1- \lambda_{\max}}\Big) \frac{1}{n},
\end{align*}
where the last equality follows since \(\sum_{y \in [A^2O^2]} \bm{d}_y = 1\). From this last relation, by taking the square root of both sides, we are able to derive the result reported in the claim.
\end{proof}

\subsection{Concentration Result When starting from Non-stationarity}
The assumption of starting from a distribution that coincides with the stationary one may be somewhat restrictive. In this section, we will extend the result of Lemma~\ref{lemma:distributionBoundStat} to also include these cases. In particular, we define the new lemma and provide theoretical guarantees for it.
\begin{restatable}[]{lemma}{distributionBoundNonStat}\label{lemma:distributionBoundNonStat}
	Let us assume to have a POMDP instance \(\mathcal{Q}\) with a finite number of states \(\mathcal{S}\) and a finite number of actions \(\mathcal{A}\). Let \(\pi \in \mathcal{P}\) be the belief-based policy used on \(\mathcal{Q}\) and let us assume that the induced stationary distribution \(\bm{d}_{AS} \in \Delta(\mathcal{A} \times \mathcal{S})\) is ergodic, thus having that \(\bm{d}_{AS}(a,s) > 0 \; \forall a \in \mathcal{A}, s \in \mathcal{S}\). Let Assumption~\ref{ass:weaklyRev} hold. If \(\bm{v} \in \Delta(\mathcal{A} \times \mathcal{S})\) is the initial distribution, then by using the estimation procedure defined in Algorithm~\ref{alg:estimationAlgorithm}, with probability at least \(1 - \delta\), it holds that:
    \begin{align*}
        \|\bm{d}_{A^2O^2} - \widehat{\bm{d}}_{A^2O^2}\|_2 \le \sqrt{\Big( \frac{1+\lambda_{\max}}{1-\lambda_{\max}}\Big)\frac{(2C + 5\log(C/\delta))}{2 n}},
    \end{align*}
    where \(\lambda_{\max}\) denotes the maximum spectral gap associated with the inhomogeneous chain, \(C:=C(\bm{\nu}, \bm{d}_{AS})\) is a constant whose value depends both on the initial distribution \(\bm{\nu}\) and the stationary distribution \(\bm{d}_{AS}\), while \(n\) denotes the number of tuples of the form \((a,a',o,o')\) collected during the interaction with the environment.
\end{restatable}
\begin{proof}
As in Lemma~\ref{lemma:distributionBoundStat}, we combine previous results appearing in~\citet{kontorovic2012uniform} with new concentration results from~\citet{fanHoeffding2021}. In particular, we will use the result appearing in in~\citet{fanHoeffding2021} (see their Theorem 12) which defines concentration results of samples coming from a non-stationary Markov chain. By applying this result to vectors, we get the following relation:
\begin{align}\label{eq:probActionDistNonStat}
    \mathbb{P}_{\bm{\nu}} \big( \|\bm{d}_{A^2O^2} - \widehat{\bm{d}}_{A^2O^2}\|_2 > \mathbb{E}_{\bm{\nu}} \|\bm{d}_{A^2O^2} - \widehat{\bm{d}}_{A^2O^2}\|_2 + \epsilon \big) \le C(\bm{\nu}, \bm{d}_{AS}, p) \exp \Big(- \frac{1}{q} \; \frac{1 - \lambda_{\max}}{1 + \lambda_{\max}} \; 2 n \epsilon^2\Big)
\end{align}
where \(C(\bm{\nu}, \bm{d}_{AS}, p)\)\footnote{For a precise definition of \(C(\bm{\nu}, \bm{d}_{AS}, p)\) we refer to Theorem 12 in~\citet{fanHoeffding2021}.} is a finite constant depending both on the initial and stationary state distribution while \(p\) and \(q\) are such that \(1/p+1/q=1\).
In the expression we used \(\mathbb{P}_{\bm{\nu}}\) and \(\mathbb{E}_{\bm{\nu}}\) to emphasize that the initial distribution is \(\bm{\nu}\).\\
Generally, for our purposes, we are interested in the real stationary distribution value, without the bias induced when starting from an arbitrary initial distribution. Indeed, we recall that vector \(\bm{d}_{A^2O^2} = \mathbb{E}_{\bm{d}_{AS}} \Big[ \widehat{\bm{d}}_{A^2O^2} \Big]\), thus being different from 
\(\mathbb{E}_{\bm{\nu}} \Big[\widehat{\bm{d}}_{A^2O^2}\Big]\) when \(\bm{\nu} \neq \bm{d}_{AS}\).\\
To proceed with the analysis, we need to bound \(\mathbb{E}_{\bm{\nu}} \|\bm{d}_{A^2O^2} - \widehat{\bm{d}}_{A^2O^2}\|_2\) which can be done by using Lemma~\ref{lemma:expectNorm2diffBoundNonStat}. We are indeed able to show that:
\begin{align*}
    \mathbb{E}_{\bm{\nu}} \Big[ \|\bm{d}_{A^2O^2} - \widehat{\bm{d}}_{A^2O^2}\|_2 \Big] \le \sqrt{\Big(\frac{1 + \lambda_{\max}}{1 - \lambda_{\max}}\Big) \frac{C(\bm{\nu}, \bm{d}_{AS})}{n}},
\end{align*}
where \(C(\bm{\nu}, \bm{d}_{AS})=\left\| \frac{\bm{\nu}}{\bm{d}_{AS}} \right\|_{\infty}\), with \(\frac{\bm{\nu}}{\bm{d}_{AS}}\) being the element-wise ratio of the two vectors\footnote{The value of \(C(\bm{\nu}, \bm{d}_{AS})\) corresponds to that of \(C(\bm{\nu}, \bm{d}_{AS}, p)\) when \(p=\infty\) and \(q=1\).}. By substituting these values into Equation~\eqref{eq:probActionDistNonStat}, we get:
\begin{align*}
    \mathbb{P}_{\bm{\nu}} \big( \|\bm{d}_{A^2O^2} - \widehat{\bm{d}}_{A^2O^2}\|_2 > \sqrt{\Big(\frac{1 + \lambda_{\max}}{1 - \lambda_{\max}}\Big) \frac{C(\bm{\nu}, \bm{d}_{AS})}{n}} + \epsilon \big) \le C(\bm{\nu}, \bm{d}_{AS}) \exp \Big(- \frac{1 - \lambda_{\max}}{1 + \lambda_{\max}} \; 2 n \epsilon^2\Big).
\end{align*}
By setting the right-hand side of the bound equal to \(\delta\), and rearranging the terms, we are able to get the final result of the Lemma.
\end{proof}

\begin{lemma}\label{lemma:expectNorm2diffBoundNonStat}
Under the same conditions of Lemma~\ref{lemma:expectNorm2diffBound} but assuming to start from an arbitrary distribution \(\bm{\nu} \in \Delta(\mathcal{A} \times \mathcal{S})\), we have:
\begin{align*}
    \mathbb{E}_{\bm{\nu}} \Big[ \|\bm{d}_{A^2O^2} - \widehat{\bm{d}}_{A^2O^2}\|_2 \Big] \le \sqrt{\Big(\frac{1 + \lambda_{\max}}{1 - \lambda_{\max}}\Big) \frac{C(\bm{\nu}, \bm{d}_{AS})}{n}},
\end{align*}
where \(C(\bm{\nu}, \bm{d}_{AS}):= \left\|\frac{\bm{\nu}}{\bm{d}_{AS}}\right\|_\infty\) and \(\frac{\bm{\nu}}{\bm{d}_{AS}}\) is a vector of size \(AS\) containing the element-wise ratio of the elements of the two vectors.
\end{lemma}
\begin{proof}
The proof of this lemma derives from that of Lemma~\ref{lemma:expectNorm2diffBound} but needs to account for the fact that the initial state distribution is arbitrary. We will use steps inspired by those appearing in~\citet{fanHoeffding2021}.\\
As in Lemma~\ref{lemma:expectNorm2diffBound}, to simplify notation, we will remove the \(A^2O^2\) pedices when referring to the stationary distribution \(\bm{d}_{A^2O^2}\). 
The analysis proceeds as follows:
\begin{align*}
    \mathbb{E}_{\bm{\nu}} \Big[\|\bm{d} - \widehat{\bm{d}}\|_2\Big]^2 & \le \mathbb{E}_{\bm{\nu}} \big[ \|\bm{d} - \widehat{\bm{d}}\|_2^2\big ] & \notag \\
    & = \mathbb{E}_{\bm{\nu}} \Big[ \sum_{y \in [A^2O^2]} |\bm{d}_y - \widehat{\bm{d}}_y |^2\Big] &  \notag \\
    & = \sum_{y \in [A^2O^2]} \mathbb{E}_{\bm{\nu}} \Big[ (\bm{d}_y - \widehat{\bm{d}}_y )^2\Big] & \notag \\
    & = \sum_{y \in [A^2O^2]} \mathbb{E}_{\bm{\nu}} \Big[\Big( \bm{d}_y - \frac{1}{n}\sum_{i =1}^{n}\mathds{1}\{Y=y\}(A_i,S_i) 
    \Big)^2\Big] &  \notag \\
    & = \sum_{y \in [A^2O^2]} \mathbb{E}_{\bm{d}_{AS}} \Big[\frac{\bm{\nu}(A_1,S_1)}{\bm{d}_{AS}(A_1,S_1)} \Big( \bm{d}_y - \frac{1}{n}\sum_{i =1}^{n}\mathds{1}\{Y=y\}(A_i,S_i) 
    \Big)^2\Big] & \qquad \qquad \text{[Change of Measure]} \notag \\
    & \le \sum_{y \in [A^2O^2]} \mathbb{E}_{\bm{d}_{AS}} \Big[\Big(\frac{\bm{\nu}(A_1,S_1)}{\bm{d}_{AS}(A_1,S_1)}\Big)^p\Big]^{1/p} \mathbb{E}_{\bm{d}_{AS}} \Big[\Big(\bm{d}_y - \widehat{\bm{d}}_y \Big)^{2q}\Big]^{1/q} & \qquad \qquad \text{[H\"older]} \notag \\
    & = \left\| \frac{\bm{\nu}}{\bm{d}_{AS}} \right\|_{\infty} \sum_{y \in [A^2O^2]} \mathbb{E}_{\bm{d}_{AS}} \Big[\Big(\bm{d}_y - \widehat{\bm{d}}_y \Big)^2\Big] & \qquad \qquad \text{[\(p=\infty\), \(q=1\)]} \notag \\
    & \le \left\| \frac{\bm{\nu}}{\bm{d}_{AS}} \right\|_{\infty} \Big( \frac{1 + \lambda}{1 - \lambda}\Big) \; \frac{1}{n} & \qquad \qquad \text{[Lemma~\ref{lemma:expectNorm2diffBound}]} \notag \\
    & \le C(\bm{\nu}, \bm{d}_{AS}) \: \Big( \frac{1 + \lambda}{1 - \lambda}\Big)\; \frac{1}{n}, \notag
\end{align*}
where in the derivation we made explicit the dependence of each term \(\bm{d}_y\) from the sequence of elements \(((A_i,S_i))_{i \in [n]}\).\\ 
To clarify to which tuple the elements \((A_i,S_i)\) refer, we recall that the tuples used for the estimation of \(\bm{d}\) are \(((a_i,a'_i,o_i,o'_i))_{i \in [n]}\). This sequence is itself originated by an underlying sequence \(((A_i,a'_i,S_i,s'_i))_{i \in [n]}\), where we used the upper case to denote the terms appearing in the proof above.\\
From the obtained result, we can see that starting from a distribution that is different from the stationary one leads to a further multiplicative term \(C(\bm{\nu}, \bm{d}_{AS})\) in the final concentration result. This term explains how much the starting distribution differs from the stationary one.
\end{proof}

\vspace{0.2cm}
\section{Useful Results for Belief MDPs}
In this section, we will present a relevant result reported in~\citet{zhou2021regime} that bounds the bias span in a belief MDP. Then, we will present another core result coming from~\citet{deCastroConsistent2017} that allows bounding the belief error with respect to the estimation error of an HMM model.\\
We start by reporting the average reward Bellman expectation equation for belief MDPs under a belief-based policy \(\pi\):
\begin{align}\label{appendix:BellExpectation}
	\rho^{\pi} + v_{\pi}(b)\; = \underset{a \sim \pi(\cdot|b)}{\mathbb{E}}\left[g(b,a) + \int_{\mathcal{B}}P(\,db'|b, a) v_{\pi}(b') \right].
\end{align}
We report as well the average reward Bellman optimality equation for belief MDPs:
\begin{align}\label{appendix:BellOptimality}
	\rho^* + v(b)\; = \underset{a \in \mathcal{A}}{\max}\left[g(b,a) + \int_{\mathcal{B}}P(\,db'|b, a) v(b') \right].
\end{align}

Various conditions allow the two equations to be satisfied. One condition requires the MDP to be weakly communicating~\cite{bertsekas1995dynamic}. In particular, this condition is verified under Assumption~\ref{ass:minElem}, as also shown in ~\citet{zhou2021regime}. In particular, they derive a result that bounds the span of the bias function under this assumption:

\begin{proposition}[Uniform bound on the bias span from~\citet{zhou2021regime}]\label{prop:uniformBoundBias}
	Let us assume to have a POMDP instance that can be rewritten as a belief MDP. If Assumption~\ref{ass:minElem} holds, then for \(\rho,v\) satisfying the Bellman optimality equation~\ref{appendix:BellOptimality}, we have the span of the bias function \(span(v):=\max_{b \in \mathcal{B}}v(b)-\min_{b \in \mathcal{B}}v(b)\) is bounded by \(D(\epsilon)\), where:
	\begin{align*}
		D(\epsilon):= \frac{8 \Big(\frac{2}{(1 -\alpha)^2} + (1 + \alpha)\log_{a} \big( \frac{1 - \alpha}{8}\big) \Big)}{1 - \alpha}, \qquad \text{with} \qquad \alpha = \frac{1 - 2\epsilon}{1 - \epsilon} \in (0,1).
	\end{align*}
\end{proposition}
For all \(v_k\) associated with the estimated optimistic belief MDP \(Q^{(k)}\), this proposition guarantees that \(span(v_k)\) is bounded by \(D=D(\epsilon/2)\) uniformly in \(k\). 
It is possible to see that the bound stated in Proposition~\ref{prop:uniformBoundBias} can be applied as well for the Bellman expectation Equation~\eqref{appendix:BellExpectation}.\\

The second set of results is instead related to controlling the belief error using the estimation errors of model parameters. We first report the result used in~\citet{zhou2021regime} and taken from~\citet{deCastroConsistent2017}. Then, we will report a related result that is more suitable for our setting.

\begin{proposition}(Controlling the belief error~\cite{zhou2021regime})\label{prop:beliefBoundOrig}
	Assume to have a transition matrix \(\mathbb{P}\) of size \(S^2\) with minimum entry \(\epsilon > 0\). Given \((\widehat{\bm{\mu}}, \widehat{\mathbb{P}})\), an estimator of the true model parameters \((\bm{\mu}, \mathbb{P})\). For an arbitrary reward-action sequence \(\{r_{1:t}, a_{i:t}\}_{t\ge 1}\), let \(\widehat{b}_t\) and \(b_t\) be the corresponding beliefs in period \(t\) under \((\widehat{\bm{\mu}}, \widehat{\mathbb{P}})\) 
	and \((\bm{\mu}, \mathbb{P})\) respectively. Then there exist constants \(L_1\) and \(L_2\) such that:
	\begin{equation*}
		\lVert \widehat{b}_t - b_t\rVert_1 \le L_1 \lVert \widehat{\bm{\mu}} - \bm{\mu}\rVert_1 + L_2 \lVert \widehat{\mathbb{P}} - \mathbb{P}\rVert_F,
	\end{equation*}
	where  \(L_1 = 4S (\frac{1 - \epsilon}{\epsilon})^2/ \min\{\bm{\mu}_{\min}, 1 - \bm{\mu}_{\max}\}\), \(L_2 = \frac{4S(1 - \epsilon)^2}{\epsilon^3} + \sqrt{S}\), \(\lVert\cdot\rVert_F\) is the Frobenius norm, \(\bm{\mu}_{\max}\) and \(\bm{\mu}_{\min}\) are the maximum and minimum element of the matrix \(\bm{\mu}\) respectively.
\end{proposition}
The presented proposition is stated for Bernoulli rewards and for a single transition matrix \(\mathbb{P} \in \mathbb{R}^{S \times S}\), while in our case we have multiple rewards and different transition matrices, one for each action taken. In the following Proposition, we will show how this result can be extended for our setting, where we do not need to estimate the observation matrix since it is already known.\\ 

\begin{proposition}(Controlling the belief error in POMDPs with Known Observation Model)\label{prop:beliefErrorPOMDP}
	Let us assume to have a POMDP instance \(\mathcal{Q}\) with a finite number of states \(\mathcal{S}\) and a finite number of actions \(\mathcal{A}\). Let \(\mathbb{T}=\{\mathbb{T}_a\}_{a \in \mathcal{A}}\) represent the transition model with \(\mathbb{T}\) being a matrix of size \(SA \times S\) and each \(\mathbb{T}_a\) be one of its submatrices of dimension \(S \times S\) related to action \(a\). Let \(\mathbb{O}=\{\mathbb{O}_a\}_{a \in \mathcal{A}}\) be the known observation model determining observations in a space \(o \in \mathcal{O}\). Let now \(\widehat{\mathbb{T}}=\{\widehat{\mathbb{T}}_a\}_{a \in \mathcal{A}}\) be the estimated transition model. For any arbitrary action-observation sequence \((a_t,o_t)_{t\ge 1}\), let \(\widehat{b}_t\) and \(b_t\) be the corresponding beliefs in period \(t\) under \((\widehat{\mathbb{T}}, \mathbb{O})\) and \((\mathbb{T}, \mathbb{O})\) respectively. Then, there exists constant \(L\) such that:
	\begin{equation*}
		\lVert \widehat{b}_t - b_t\rVert_1 \le L \:\: \underset{a \in \mathcal{A}}{\max}\lVert \widehat{\mathbb{T}}_a - \mathbb{T}_a\rVert_F,
	\end{equation*}
	where \(L = \frac{4S(1 - \epsilon)^2}{\epsilon^3} + \sqrt{S}\) and \(\lVert\cdot\rVert_F\) is the Frobenius norm.
\end{proposition}

\vspace{0.2cm}
\section{Miscellanea of Useful Results}\label{appendix:usefulLemmas}
This section is devoted to the presentation of different useful results that are used throughout the work.

\begin{lemma}[Lemma A.1 in ~\citet{ramponiTruly2020}]\label{lemma:trulyBatch}
	Let \(\bm{x}, \bm{y} \in \mathbb{R}^d\) any pair of vectors, then it holds that:
	\begin{equation*}
		\left\lVert \frac{\bm{x}}{\lVert \bm{x} \rVert_2} - \frac{\bm{y}}{\lVert \bm{y} \rVert_2} \right\rVert_2 \le \frac{2 \lVert \bm{x} -\bm{y} \rVert_2}{\max\{ \lVert\bm{x}\rVert_2,\lVert\bm{y}\rVert_2\}}
	\end{equation*}
	\begin{proof}
		The presented result follows from a sequence of algebraic manipulations:
		\begin{align*}
			\left\lVert \frac{\bm{x}}{\lVert \bm{x} \rVert_2} - \frac{\bm{y}}{\lVert \bm{y} \rVert_2} \right\rVert_2 = & \left\lVert \frac{\bm{x}}{\lVert \bm{x} \rVert_2} - \frac{\bm{y}}{\lVert \bm{y} \rVert_2} \pm \frac{\bm{y}}{\lVert \bm{x} \rVert_2}\right\rVert_2\\
			= \; & \left\lVert \frac{\bm{x} - \bm{y}}{\lVert \bm{x} \rVert_2} - \frac{\bm{y}(\lVert \bm{y}\rVert_2 - \lVert\bm{x}\rVert_2)}{\lVert \bm{y} \rVert_2 \lVert \bm{x} \rVert_2} \right\rVert_2\\
			\le \; & \frac{\lVert \bm{x} -\bm{y} \rVert_2}{\lVert \bm{x} \rVert_2} + \frac{|\lVert \bm{x}\rVert_2 - \lVert\bm{y}\rVert_2|}{\lVert \bm{x}\rVert_2}\\
			\le \; & 2\frac{\lVert \bm{x} -\bm{y} \rVert_2}{\lVert \bm{x}\rVert_2},
		\end{align*}
		where the triangular inequality has been applied in the third line and the reverse triangular inequality in the last one, i.e. \(|\lVert \bm{x}\rVert_2 - \lVert\bm{y}\rVert_2| \le \lVert \bm{x} -\bm{y} \rVert_2\). The result in the lemma can be derived by observing that, for symmetry reasons, the same derivation can be performed by getting \(\lVert \bm{y}\rVert_2\).
	\end{proof}
\end{lemma}

The following result instead shows the relation between vectors which are obtained by the aggregation of higher-dimensional ones.
\begin{lemma}[Aggregation Lemma]\label{lemma:aggregation}
	Let \(\bm{M}\) be a matrix of dimension \(X \times Y\) and have positive values. Let \(\widehat{\bm{M}}\) be an estimation of \(\bm{M}\). Let now \(\bm{c}\) be a vector of dimension \(X\) obtained by summing all the elements of \(\bm{M}\) along the second dimension, such that \(\bm{c}(i) = \sum_{j=1}^J \bm{M}(i,j)\) and let \(\widehat{\bm{c}}\) be a vector obtained with the same procedure from \(\widehat{\bm{M}}\). Then we will have:
	\begin{align*}
		\|\widehat{\bm{c}} - \bm{c}\|_2 \le \sqrt{Y} \|\widehat{\bm{M}} - \bm{M}\|_F
	\end{align*}
	\begin{proof}
		Let us first consider the vectorized version of matrix \(\bm{M}\) and denote it with \(\bm{m}\) having size \(X\cdot Y\). Analogously, we define \(\widehat{\bm{m}}\).
		Let us now define the transformation from \(\bm{m}\) to \(\bm{c}\) using vector notation. We will have:
		\begin{align*}
			\bm{c} = \bm{A} \bm{m},
		\end{align*}
		where \(\bm{A}\) is a matrix of size \(X \times X \cdot Y\) and it is defined as \(\bm{A}:= \mathcal{I}_X \otimes \mathds{1}_Y^\top\) where \(\mathcal{I}_X\) is the identity matrix of size \(X\), \(\mathds{1}_Y\) is a vector of ones of dimension \(Y\), while the symbol \(\otimes\) represents the Kronecker product between the two.
		This allows us to write:
		\begin{align}
			\|\widehat{\bm{c}} - \bm{c}\|_2 & = \|\bm{A} (\widehat{\bm{m}} - \bm{m})\|_2 \notag \\
			& \le \|\bm{A}\|_2 \|\widehat{\bm{m}} - \bm{m}\|_2 \label{agg:01}\\
			& = \sigma_{\max}(\bm{A}) \|\widehat{\bm{m}} - \bm{m}\|_2 \notag\\
			& = \sigma_{\max}(\mathcal{I}_X) \sigma_{\max}(\mathds{1}_Y) \|\widehat{\bm{m}} - \bm{m}\|_2 \label{agg:02}\\
			& = \lambda_{\max}^{1/2}(\mathds{1}_Y^\top \mathds{1}_Y) \|\widehat{\bm{m}} - \bm{m}\|_2 = \sqrt{Y} \|\widehat{\bm{m}} - \bm{m}\|_2 \label{agg:03}
		\end{align}
		where the first inequality in line~\ref{agg:01} follows from the consistency property of the matrices. Line~\ref{agg:02} follows from the definition of \(\bm{A}\) and the properties of the Kronecker product, while the first equality on line~\ref{agg:03} follows from the fact that \(\sigma_{\max}(\mathcal{I}_X) = 1\) and \(\sigma_{\max}^2(\mathds{1}_Y) = \lambda_{\max}(\mathds{1}_Y^\top \mathds{1}_Y)\).\\
		The proof is concluded by the equivalence of \(\|\widehat{\bm{M}} - \bm{M}\|_F\) and its vectorized version \(\|\widehat{\bm{m}} - \bm{m}\|_2\).
	\end{proof}
\end{lemma}

\vspace{0.2cm}
\section{Simulation Details}\label{appendix:simulationDetails}
In this section, we provide details on the experiments reported in the main paper. All the reported experiments have been run on an notebook with CPU Intel i7-11th and 16G RAM using Ubuntu as OS.

\subsection{Estimation Error of Transition Matrix}
The number of states, actions, and observations of the instances used for the experiments of the estimation error of the transition model are reported in the legend of Figure~\ref{fig:estimationError}. Given each instance considered in the experiment, the transition and observation matrices have been generated as follows. An initial version of transition and observation matrices is generated with random elements and, subsequently:
\begin{enumerate}
    \item regarding the transition matrix, the experiments are reported for instances having transition matrices where the minimum transition probability satisfies $\epsilon=1/5S$, with $S$ being the number of states.
    \item regarding the observation matrix, for each pair of states and actions, we choose a specific observation that will be drawn with higher probability, in order to avoid having too much stochasticity in the reward distributions and ensure a diverse observation distribution among states. We require that the generated observation matrix has a minimum singular value lower bounded by a quantity $\alpha>=0.001$, thus satisfying our Assumption~\ref{ass:weaklyRev}.
    \item The policy class is characterized by a minimum action sampling probability $\iota=1/10A$ with $A$ being the number of actions.
\end{enumerate}

The experiment is conducted in the following way. A POMDP instance is generated by giving as input the number of states, number of actions, and number of observations and by following the procedures highlighted in points 1. and 2.\\

Samples are collected using a belief-based policy $\pi \in \mathcal{P}$ that updates its belief using the real observation matrix and a different transition matrix randomly generated, as specified in point 1. This transition matrix is not updated but kept fixed during the whole experiment. We highlight that this transition matrix is only used by the belief-based policy to induce a specific action-state distribution \(\bm{d}_{AS}\) and is indeed independent of the real transition model that needs to be estimated.\\
The minimum action probability of the belief-based policy \(\pi\) is defined as in point 3.\\

The policy collects samples and updates its belief based on its internal transition matrix, while the underlying process follows the dynamics induced by the real transition model. The collected samples are then provided to the OAS algorithm (Algorithm~\ref{alg:estimationAlgorithm}) which gives as output an estimate of the real transition model.


\subsection{Regret Experiments}
The experimental results on the regret have been conducted as described in the following.\\

Because of the computational complexity of finding the optimal policy that maximizes the long-run average reward in a known environment, we conducted the experiment by discretizing the belief space, allowing us to apply common techniques used for finite state spaces. We are thus able to solve the optimal Bellman equation appearing in~\ref{eq:ourBellmanEq}, where the computation of the optimal policy and the optimistic transition model is done by adapting the Extended Value Iteration algorithm~\cite{jaksch2010near} to the discretized state space.

The uncertainty coming from the estimates of the transition model is defined by the theoretical results obtained for the OAS algorithm. As commonly done in many practical approaches, the theoretical bounds are replaced by much smaller values. We recall that using scaled values rather than those suggested by theory mostly translates into a regret with bigger multiplicative constants or holding with smaller probability. This approach is common when performing experimental comparisons in these settings (as observed in Remark 3 in~\citet{Azizzadenesheli2016Reinforcement}).

\paragraph{SEEU algorithm~\cite{xiong2022sublinear}} In the experiments, we do not use the classical Spectral Decomposition approach, but we help the estimation procedure in the following way:
\begin{itemize}
    \item The matrices used by the Spectral Decomposition approach are provided with more information with respect to what is commonly done. In particular, instead of updating the matrices based on the observation seen when pulling an arm, we directly provide the matrix with the probabilities defining the observation distribution of the pulled arm. This caveat helps the estimation of the transition matrix and that of the observation model since vanishes the noise given by the realizations of the observations.
    \item The computation of the optimistic policy for the SEEU algorithm is done by using the estimate of the transition model but is provided with the real observation model, as also done for the OAS-UCRL algorithm.
\end{itemize}
The parameters used for the SEEU algorithm are $\tau_1 = 20000$ and $\tau_2 = 65000$ which are used to determine the length of the exploration and the exploitation phase respectively.

\paragraph{PSRL-POMDP~\cite{jahromi2022online}} Since the authors present a computationally intractable update of the model parameters, we opted for an implementation that makes use of the particle filter approach, commonly used in the Bayesian setting. We recall that the authors provide theoretical guarantees for this algorithm under the assumption that the employed estimator is consistent. However, in practice, there are no provably consistent estimators based on a Bayesian approach.\\
We implement their algorithm 1 by using the particles to represent the prior distributions and, referring to the parameters of the PSRL-POMDP algorithm, we set:
\begin{itemize}
    \item $SCHED(t_k, T_{k-1}) = t_k + T_{k-1}$ with $t_k$ representing the length of the $k$-th episode, as suggested in their work.
    \item $\widetilde{m}_t(s,a)=n_t(a)$ with $\widetilde{m}_t(s,a)$ being an upper bound to the expected number of times the pair $(s,a)$ has been encountered up to time $t$. $n_t(a)$ instead counts the number of times action $a$ has been pulled up to time $t$. This choice of \(\widetilde{m}_t(s,a)\) is suggested in their work.
    \item $N=100$ particles have been used in the experiments, while updates of the particles are triggered when the \emph{effective sample size} (ESS) associated with their weights goes below $30$.\\ 
    The choice of these values is performed by testing different hyperparameters both for the number \(N\) of particles and the effective sample size (ESS). This combination has the best results while also not requiring extensive computational time, which is a typical problem of Particle-based approaches when the number of used particles gets high.
\end{itemize}

\end{document}